\documentclass{theoretics}

\title{Smoothed Analysis for Learning Concepts with Low Intrinsic Dimension}

\ThCSauthor[aff1]{Gautam Chandrasekaran}{gautamc@cs.utexas.edu}[0009-0002-2443-8675]
\ThCSauthor[aff1]{Adam R. Klivans}{klivans@cs.utexas.edu}[0000-0001-6960-2235]
\ThCSauthor[aff1]{Vasilis Kontonis}{vasilis@cs.utexas.edu}[0000-0002-2229-1879]
\ThCSauthor[aff2]{Raghu Meka}{raghum@cs.ucla.edu}[0009-0004-7676-2762]
\ThCSauthor[aff1]{Konstantinos Stavropoulos}{kstavrop@cs.utexas.edu}[0009-0006-2643-4298]
\ThCSaffil[aff1]{University of Texas at Austin} 
\ThCSaffil[aff2]{University of California, Los Angeles} 
\ThCSthanks{
Invited from COLT 2024~\cite{DBLP:conf/colt/ChandrasekaranK24}. 
Gautam Chandrasekaran is supported by the NSF AI Institute for Foundations of Machine Learning (IFML).
Adam R. Klivans is supported by NSF award AF-1909204 and the NSF AI Institute for Foundations of Machine Learning (IFML).
Vasilis Kontonis is supported by the NSF AI Institute for Foundations of Machine Learning (IFML). 
Raghu Meka is supported by NSF Collaborative Research: Award 2217033.
Konstantinos Stavropoulos is supported by the NSF AI Institute for Foundations of Machine Learning (IFML) and by scholarships from Bodossaki Foundation and Leventis Foundation.
}
\ThCSshortnames{G.~Chandrasekaran, A.~R.~Klivans, V.~Kontonis, R.~Meka, K.~Stavropoulos}
\ThCSshorttitle{Smoothed Analysis for Learning Concepts with Low Intrinsic Dimension} 
\ThCSyear{2026}
\ThCSarticlenum{13}
\ThCSreceived{May 7, 2025}
\ThCSrevised{Feb 1, 2026}
\ThCSaccepted{Mar 22, 2026}
\ThCSpublished{Sep 3, 2026}
\ThCSdoicreatedtrue
\ThCSkeywords{PAC Learning; Agnostic Learning; Beyond Worst Case Analysis; Margin; Halfspace; Geometric Concepts; Gaussian Surface Area}

\usepackage{empheq}
\usepackage{dsfont}
\usepackage{aliascnt} 
\usepackage{booktabs}
\usepackage{bbm}

\usepackage{tikz}
\usetikzlibrary{patterns}
\usetikzlibrary{arrows,shapes,automata,backgrounds,petri,positioning}
\usetikzlibrary{shadows}
\usetikzlibrary{calc}
\usetikzlibrary{spy}
\usetikzlibrary{angles, quotes}
\usetikzlibrary{matrix}
\usepackage{pgf,pgfplots}
\usepackage{pgfmath,pgffor}
\pgfplotsset{compat=1.17}

\usepackage{framed}

\usepackage[capitalise,noabbrev,nameinlink]{cleveref}

\newcommand{\eqdef}{\coloneqq}

\newcommand\twonorm[1]{\|#1\|_2}
\newcommand\norm[1]{\left\| #1 \right\|}

\renewcommand\vec[1]{\mathbf{#1}}
\DeclareMathOperator*{\pr}{\mathbf{Pr}}
\DeclareMathOperator*{\E}{\mathbf{E}}
\newcommand{\proj}{\mathrm{proj}}

\def\d{\mathrm{d}}
\newcommand{\normal}{\mathcal{N}}

\def\multichoose#1#2{\ensuremath{\left(\kern-.3em\left(\genfrac{}{}{0pt}{}{#1}{#2}\right)\kern-.3em\right)}}

\newcommand{\e}{\mathbf{e}}

\newcommand{\err}{\mathrm{err}}

\newcommand{\R}{\mathbb{R}}

\newcommand{\Z}{\mathbb{Z}}
\newcommand{\N}{\mathbb{N}}
\newcommand{\eps}{\epsilon}

\newcommand{\poly}{\mathrm{poly}}

\newcommand{\sign}{\mathrm{sign}}

\newcommand{\opt}{\mathrm{opt}}
\newcommand{\D}{D}

\newcommand{\Ind}{\mathds{1}}
\newcommand{\1}{\Ind}

\newcommand{\wt}{\widetilde}

\newcommand{\smoothopt}{\mathrm{opt}_{\sigma}}

\newcommand{\x}{\vec x}
\renewcommand{\vv}{\vec v}
\newcommand{\vu}{\vec u}
\newcommand{\vw}{\vec w}
\newcommand{\z}{\vec z}

\newcommand{\lamdba}{\lambda}

\newcommand{\F}{\mathcal{F}}
\newcommand{\Gauss}{\mathcal{N}}

\newcommand{\Dmarginal}{D_{\x}}
\newcommand{\Djoint}{D}

\newcommand{\Alg}{\mathcal{A}}
\newcommand{\Dclass}{\mathbb{D}}
\newcommand{\ouoperator}{T}

\newcommand{\ind}{\mathds{1}}

\LinesNotNumbered
\begin{document}
\maketitle

\begin{abstract}%

In traditional models of supervised learning, the goal of a learner-- given examples from an arbitrary joint distribution on $\mathbb{R}^d \times \{\pm 1\}$-- is to output a hypothesis that is competitive (to within $\epsilon$) of the best fitting concept from some class.  In order to escape strong hardness results for learning even simple concept classes, we introduce a smoothed-analysis framework that requires a learner to compete only with the best classifier that is robust to small random Gaussian perturbation.

This subtle change allows us to give a wide array of learning results for any concept that (1) depends on a low-dimensional subspace (aka multi-index model) and (2) has a bounded Gaussian surface area.  This class includes functions of halfspaces and (low-dimensional) convex sets, cases that are only known to be learnable in non-smoothed settings with respect to highly structured distributions such as Gaussians.

Our definition of smoothed agnostic learning is an interpolation between the case where the instance distribution $D$ and the optimal classifier can be arbitrarily coupled (which corresponds to agnostic learning and $\sigma = 0$) and completely decoupled (when $\sigma = \infty$).  This decoupling allows us to avoid worst-case concepts that can encode complexity-theoretic primitives.

Surprisingly, our analysis also yields new results for traditional non-smoothed frameworks such as learning with margin.  In particular, we obtain the first algorithm for agnostically learning intersections of $k$-halfspaces in time  $k^{\poly(\frac{\log k}{\epsilon \gamma}) }$ where $\gamma$ is the margin parameter.  
Before our work, the best-known runtime was exponential in $k$ (Arriaga and Vempala, FOCS' 99).

\end{abstract}

\section{Introduction}

In either the PAC or agnostic learning model~\cite{val84,Haussler:92,KSS:94}, a learner is given access to random labeled examples and has to compute a classifier that performs approximately as well as the best classifier in a target concept class (in the PAC model it is further assumed that the best classifier has zero error).  More precisely, for an instance distribution $D$ over 
$\R^d \times \{ \pm 1\}$ and a concept class $\mathcal F$, the optimal error is 
defined as $\opt = \inf_{f \in \mathcal F} \pr_{(\x, y) \sim D}[f(\x) \neq y]$.
Without assumptions about the feature distribution and/or the label generating process, learning is known to be computationally hard for even the simplest concept classes \cite{Khar1993,  GR:06,DBLP:journals/toc/Dachman-SoledLMSWW09, DBLP:journals/jcss/KhotS11, DBLP:journals/siamcomp/FeldmanGRW12,klivans2009cryptographic,DOSW:11, FLS:11colt, DanielyV21}. In particular, learning halfspaces (linear classifiers) is intractable without strong distributional assumptions \cite{KKMS:05,GR:06, Feldman:06, Daniely16}.

In order to bypass these hardness results, a large body of research has focused on taking assumptions on the data-generating distribution. 
The most common approaches are (1) making a distributional assumption on the underlying feature distribution or marginal, e.g., that it is Gaussian or uniform on the hypercube \cite{LMN:93,Long:03, KKMS:05,KOS:08, GKK:08,kalai2009learning,ABL17,DKKTZ21}, or (2) assuming that the labels are not generated adversarially  
\cite{AwasthiBHU15,AwasthiBHZ16,DGT19,DKTZ20,CKMY20,ZSA20,DKKTZ22,daniely2023computational,daniely2023most}.

\paragraph{Our Smoothed Learning Model} 
In this work, we depart from those paradigms, and instead of explicitly imposing structure 
on the feature or the label distributions we simply relax the notion of optimality.  Inspired by the seminal works \cite{spielman2004smoothed,spielman2005smoothed} on the 
smoothed-complexity of algorithms,
we require the learner to compete against
the minimum possible error over classifiers that have been translated by a small Gaussian perturbation. 
 Formally, we have the following definition:
 
\begin{definition}[Smoothed Optimality]\label{definition:concept-smoothing} 
Fix $\eps, \sigma>0$ and $\delta\in (0,1)$.
Let $\F$ be a class of Boolean concepts and let $\Dclass$ be a class of distributions over $\R^d$.
    Let $\Djoint$ be a distribution over $(\x, y) \in \R^d\times \{\pm 1\}$ such that its $\x$-marginal $\Dmarginal\in \Dclass$.
    We say that the algorithm $\Alg$ learns $\F$ in the $\sigma$-smoothed setting if, after receiving i.i.d. samples from $\Djoint$, 
    $\Alg$ outputs a hypothesis $h:\R^d \to \{\pm 1\}$ such that,
    with probability at least $1-\delta$, it holds
        $\pr_{(\x, y) \sim \Djoint}[ h(\x) \neq y] \le 
        \opt_\sigma + \eps$,
        where 
    \begin{equation}\label{equation:concept-smoothing-goal}
    \opt_\sigma = 
    \inf_{f \in \mathcal F} 
    \E_{\z\sim\Gauss}
    \Big[
         \pr_{(\x, y) \sim \Djoint}[ f(\x + \sigma \z) \neq y ] 
         \Big]
         \,.
    \end{equation}
\end{definition}

We observe that by taking $\sigma = 0$ in \Cref{definition:concept-smoothing} we recover the standard definition of agnostic learning.\footnote{Note our smoothing applies to the PAC framework as well, as that is simply the case $\opt_{0} = 0$.}  On the other extreme as $\sigma \to \infty$, every concept is 
evaluated on a random input unrelated to the label $y$ and the error essentially does not depend on the 
concept $f$. The smoothed agnostic learning of \Cref{definition:concept-smoothing} is therefore an interpolation between the case where the instance distribution $D$ and the optimal classifier can be arbitrarily coupled (which corresponds to agnostic learning and $\sigma = 0$) and completely decoupled (when $\sigma = \infty$).  This decoupling allows us to avoid worst-case concepts that can encode complexity-theoretic primitives.  


\paragraph{Learning Concepts with Low Intrinsic Dimension}
We focus on the general class of concepts with low intrinsic dimension, i.e., that implicitly
depend on few relevant directions (these are also known as linear or subspace juntas \cite{vempala2011structure,de2019your,DeMN21}).
More precisely, a concept $f$ is of low intrinsic dimension if there exists
an --- \emph{unknown to the learner} --- subspace $V$ of dimension at most $k$ such that 
$f$ only depends on the projection of $\x$ onto $V$, i.e., $f(\x) = f(\proj_{V} \vec \x)$
for all $\x$. We will also use the term ``low-dimensional'' for such concepts.
Perhaps \emph{the} most well-studied low-dimensional concept
class is that of halfspaces or linear threshold functions \cite{Rosenblatt:62,MinskyPapert:88},
where $k=1$.
Another popular low-dimensional class that has been extensively studied is intersections of $k$ halfspaces \cite{DBLP:journals/jcss/BlumK97,ArriagaVempala:99,DBLP:journals/jcss/KlivansS08,KOS:08,Vempala10}.
More broadly, in \Cref{def:bounded-surface-area-concepts} we define a general class
of low dimensional concepts with ``well-behaved'' decision boundary that includes the previous mentioned classes (and more) 
as special cases.  Essentially all efficient algorithms in prior work for learning such concepts (in fact even for learning halfspaces) rely on strong assumptions, such as Gaussians \citep{KKMS:05,KOS:08}.
We investigate whether it is possible to design efficient learning 
algorithms in the smoothed setting of \Cref{def:bounded-surface-area-concepts} for 
natural concept classes while weakening the distributional assumptions that have been used so
far in the literature:
\begin{center}
\emph{
Can we relax the strong distributional assumptions (such as Gaussianity) required by previous works and still obtain comparable efficient algorithms in the smoothed setting?}
\end{center}
We answer the above question positively and show that efficient smoothed learning is possible assuming only that the feature distribution is concentrated (e.g., bounded or sub-gaussian). 
In particular, our results in the smoothed setting establish learnability under 
discrete distributions that are commonly used in hardness constructions in the standard 
agnostic setting (see, e.g., \cite{DanielyV21}).
At the same time, we show that our smoothed learning model improves and generalizes prior models such as learning with margin. In fact, for standard non-smoothed settings such as learning intersections of $k$-halfspaces with margin, we are able to obtain significant improvements over the prior works as corollaries of our smoothed learning results.

\subsection{Our Results}
\label{sec:results}
In this section we present our main contributions and discuss the connections of the smoothed learning model of \Cref{definition:concept-smoothing} with other models.

\paragraph{Measure of Complexity: Gaussian Surface Area}
As mentioned above, we require that the concept class is low-dimensional, i.e., that it depends on few relevant directions.  Moreover, we assume that it has bounded Gaussian Surface Area (GSA).  The GSA of a boolean function $f$, denoted from now on as $\Gamma(f)$, 
is defined to be the surface area of its decision boundary weighted by the Gaussian density, see \Cref{def:GSA} for a formal definition. In the context of learning theory, GSA was first used in \cite{KOS:08}  where it was shown that concepts with bounded GSA admit efficient learning algorithms under Gaussian marginals. Since then, GSA has played a significant role as a complexity measure in learning theory and related fields; see, e.g.,~\cite{Kane11, Neeman14, KTZ19, DeMN21}.  
\begin{definition}[Low-Dimensional, Bounded Surface Area Concepts]
\label{def:bounded-surface-area-concepts}
For $k \in \mathbb N$ and $\Gamma > 0$,  a concept $f:\R^d \mapsto \{\pm 1\}$ belongs in the class $\mathcal{F}(k, \Gamma)$ if:
\begin{enumerate}
\item There exists a subspace $U$ of dimension at most $k$ such that $f(\x) = f(\proj_U(\x))$.
\item The Gaussian Surface Area of $f$, $\Gamma(f)$ is at most $\Gamma$.
\item For every $\vec t \in \R^d$ and $r>0$, the function $f(r \x + \vec t) \in \mathcal F(k, \Gamma)$.
\end{enumerate}

\end{definition}
\begin{remark}
(1) While we are using GSA as a complexity measure, we stress that we do {\bf not} assume that the $\x$-marginal distribution is Gaussian.  
(2) The invariance under scaling and translation (the third property of \Cref{def:bounded-surface-area-concepts}) is a mild technical assumption that is satisfied by all classes that we have discussed so far (halfspaces and functions of halfspaces, ptfs, etc.), see also \Cref{lem:gsa_bounds}.
\end{remark}
We note that halfspaces belong in $\mathcal F(1, O(1))$, intersections of $k$ halfspaces belong in $ \mathcal F(k, O(\sqrt{\log k}))$, 
and $k$-dimensional polynomial threshold functions of degree $\ell$ in $\mathcal F(k, O(\ell))$.
Moreover, \Cref{def:bounded-surface-area-concepts} also contains non-parametric classes: 
for example, $\mathcal F(k, O(k^{1/4}))$ includes all convex bodies in $k$ dimensions,
see \Cref{lem:gsa_bounds}.
We remark that low-dimensional functions similar to those in \Cref{def:bounded-surface-area-concepts} are also referred to (usually when the functions are real-valued) as Multi-index Models (MiMs) ---  a common modeling assumption to avoid the curse of dimensionality in statistics~\cite{Friedman:1980tu, Huber85-pp, Li91, HL93,xia2002adaptive, Xia08}.

\subsubsection{Main Results: Smoothed Agnostic Learning under Concentration}

We show that we can efficiently learn assuming only concentration properties for the $\x$-marginal.  More precisely, we assume that the distribution has 
sub-gaussian tails, i.e., for every unit direction $\vec v$ it holds $\pr_{\x \sim \D_\x}[|\vec v \cdot \x| \geq t] \leq \exp(-\Omega(t^2))$. 
\begin{theorem}[Sub-Gaussian -- Informal, see also \Cref{thm:smooth_learning_subgaussian-main}]
\label{thm:intro-subgaussian}
    Let $D$ be a distribution on $\R^{d}\times \{\pm 1\}$ with sub-gaussian $\x$-marginal.
    There exists an algorithm that learns the class $\mathcal F(k, \Gamma)$ in 
    the $\sigma$-smoothed setting with 
    $N = d^{\poly(\frac{k\Gamma}{\sigma\epsilon})}\log(\frac{1}{\delta}) $ samples and $\poly(d,N)$ runtime.
\end{theorem}
We remark that our result works even under weaker tail assumptions: in particular it suffices that the tails are strictly sub-exponential, see \Cref{def:strict_subexp} and \Cref{thm:smooth_learning_subgaussian-main}.  

We observe that the runtime of \Cref{thm:intro-subgaussian} for learning a single halfspace (where $k = 1$) in the smoothed setting qualitatively matches the best known runtime for agnostic learning under Gaussian marginals.  For concepts with bounded Gaussian surface area, in \cite{KOS:08}, under the assumption that the $\x$-marginal is Gaussian, an algorithm with $d^{\poly(\Gamma/\eps)}$ runtime is given.  When the intrinsic dimension $k = O(1)$, our results in the smoothed setting achieve the same runtime and only require sub-gaussian tails.   By a simple reduction to learning parities on the hypercube, see \Cref{thm:sq-noise-variance}, we obtain a Statistical Query (SQ) lower bound of 
$d^{\Omega(\min(k, \Gamma))}$ queries (or tolerance $d^{-{\Omega(\min(k, \Gamma))}})$ for learning over sub-gaussian marginals, showing that in some cases the exponential dependency on the surface area or the intrinsic dimension to learn $\mathcal{F}(k, \Gamma)$ is unavoidable.

Our second result shows that we can significantly improve the runtime when the marginals are bounded.
Bounded marginals is a common assumption especially since it is often used together with geometric margins
assumptions.  At a high-level, in our smoothed learning setting having bounded $\|\vec x\|_2$ 
means that the ratio $\|\vec x\|_2/\sigma$ is more well behaved in the sense that the adversary, who picks $\x$, cannot overpower the smoothing noise $\sigma$ (see \Cref{definition:concept-smoothing}). Observe that 
if the adversary is allowed to select $\x$ with arbitrarily large norm, the effect of Gaussian noise in \Cref{definition:concept-smoothing} is negligible and we return to the standard agnostic setting.

\begin{theorem}[Bounded -- Informal, see also \Cref{thm:random_proj_bounded_main}]
\label{thm:intro-bounded}
    Let $D$ be a distribution on $\R^d\times \{\pm 1\}$ with $\x$-marginal bounded in the unit ball.  There exists an algorithm that learns the class $\mathcal{F}(k, \Gamma)$ in
    the $\sigma$-smoothed setting with $N = k^{\poly(\frac{\Gamma}{\epsilon \sigma})}\log(\frac{1}{\delta})$ samples and $\poly(d,N)$ runtime.
\end{theorem}

Using our theorem and bounds on the Gaussian surface area we readily obtain corollaries for specific classes.  For example, we learn efficiently intersections of $k$-halfspaces with $k^{\poly(\log k / (\sigma \eps))}$ samples and arbitrary $k$-dimensional convex bodies  with  $k^{\poly(k/(\sigma \eps) )}$ samples.

\subsubsection{Applications}
In this section we present several applications of our general smoothed learning results to standard (non-smoothed) agnostic learning settings that have been considered in the literature. 
In many cases we obtain significant improvements over the best-known results.

\paragraph{Agnostic Learning with Margin}
Our smoothed learning model is related to margin-based learning (originally defined in \cite{BS00}) because, at a high-level, it
penalizes the adversary for placing points very close to the decision boundary to create
difficult instances.  In (agnostic) learning of a class $\mathcal C$ with $\gamma$-margin 
the feature distribution is typically assumed to be bounded and the goal is to compute a classifier with error 
\begin{equation}
\label{eq:margin-opt}
\pr_{(\x,y) \sim D}[h(\x) \neq y] \leq 
\underbrace{\inf_{f \in \mathcal C}\pr_{(\x,y) \sim D}\Big[\sup_{\|\vec u\|_2 \leq \gamma} \1\{f(\x + \vec u) \neq y\} \Big]}_{\text{margin-opt}_\gamma}
+ \eps\,.
\end{equation}
We show that for any concept class with intrinsic dimension $k$, for $\sigma = \Omega(\gamma/\sqrt{k \log(1/\eps)})$, it holds
$\smoothopt \leq \text{margin-opt}_\gamma + \eps$. Therefore, any learning algorithm for 
the smoothed learning setting can be directly used to learn in the $\gamma$-margin setting.
For the special case of intersections of $k$-halfspaces we show that the gap between
$\text{margin-opt}_{\gamma}$ and $\smoothopt$ is $\eps$ by choosing 
$\sigma = \Omega(\gamma/\sqrt{\log k \log(1/\eps)})$. Using this fact and \Cref{thm:intro-bounded} we obtain the following corollary.

\begin{corollary}[Intersections of $k$-halfspaces with $\gamma$-margin]
\label{inf:margin-intersections} 
Let $D$ be a distribution on $\R^d \times \{ \pm 1\}$ whose $\x$-marginal is bounded in the unit ball and let $\mathcal C$ be the class of intersections of $k$-halfspaces.
There exists an algorithm that draws $N = k^{\poly(\log k/\gamma \eps)}\log(\frac{1}{\delta})$ samples, runs in $\poly(d, N)$ time and computes a hypothesis $h$ such that, with probability at least $1-\delta$, 
it holds $\pr_{(\x,y)\sim D}[h(\x) \neq y] \leq \mathrm{margin}$-$\mathrm{opt}_\gamma + \eps$.
\end{corollary}

We remark that, prior to our work, the best known runtime for learning intersections of
$k$-halfspaces with $\gamma$-margin in the agnostic setting from \cite{ArriagaVempala:99}
was exponential in the number of halfspaces that is $k^{\poly(\frac{k}{\gamma \eps})}$.  
In the noiseless setting, better runtimes are known \cite{klivans-servedio-margin,Gottlieb_margin,goel2019learning}.\footnote{\cite{goel2019learning} also considers the agnostic case, but their error guarantees are of the form $\sqrt{\mathrm{margin}\text{-}\mathrm{opt}_\gamma} + \eps$ as opposed to the $\mathrm{margin}\text{-}\mathrm{opt}_\gamma+\epsilon$ guarantees that we achieve.}
Beyond intersections of halfspaces with $\gamma$-margin, we obtain new results for other 
classes such as polynomial threshold functions and general convex sets, see \Cref{sec:margin} 
for more details.

\paragraph{Agnostic Learning under Smoothed Distributions}
We conclude with some applications of our framework to the (different) scenario where the {\em marginal distribution itself is smoothed}.
For example, in \cite{Kane2013LearningHalfspacesLogConcave} sub-Gaussian marginals are smoothed by additive 
Gaussian noise; i.e., for some sub-Gaussian distribution $D$ a sample from 
the smoothed distribution $D_\tau$ is generated as $\x + \tau \vec z$ for $\x \sim D$
and $\vec z \sim \normal$.  
We remind the reader that our smoothed learning model of \Cref{definition:concept-smoothing} does not try to make the $\x$-marginal more benign by a Gaussian convolution as is done in smoothed distribution learning settings \cite{kalai2008decision,kalai2009learning,Kane2013LearningHalfspacesLogConcave}.  In our model, the learner observes i.i.d.\ examples from the original marginal $D_\x$ and not 
from the convolution $D_\x + \sigma \normal$.  
Perhaps surprisingly, we show that \Cref{thm:intro-subgaussian} can be used to significantly improve the results of \cite{Kane2013LearningHalfspacesLogConcave} and other results for learning with smoothed marginals:

\begin{corollary}[Informal, see also \Cref{corollary: agnostic_smoothed_subexp}]
\label{intro-cor:smoothed-distributions}
Let $D_\tau$ be a smoothed sub-Gaussian distribution.   
There exists an algorithm that agnostically learns the class $\mathcal F(k, \Gamma)$ 
with $N = d^{\poly(\frac{k\Gamma}{\tau \epsilon})}\log(\frac{1}{\delta})$ samples and $\poly(d,N)$ runtime.
\end{corollary}

We remark that \Cref{intro-cor:smoothed-distributions} (i) generalizes the results of 
\cite{Kane2013LearningHalfspacesLogConcave} to any class of $k$-dimensional concepts with bounded surface area and (ii) yields an exponential improvement over \cite{Kane2013LearningHalfspacesLogConcave} where the runtime is doubly exponenential in $k$, i.e., $d^{\log \log(k/(\tau/\eps))^{\wt{O}(k)} \poly(1/(\tau \eps))}$.

\paragraph{Agnostic Learning under Anti-concentration}
Finally, another important direction considered in the literature is making structural
assumptions such as anti-concentration over the feature distribution.
In particular, in \cite{gollakota_testable} apart from sub-gaussian tails the distribution
is assumed to satisfy anti-concentration over slabs, i.e., for any unit vector $\vec v$ and interval $I$ it holds that $\pr_{\x \sim D_\x}[\vec v \cdot \x \in I] \leq O(|I|)$, where $|I|$
is the length of the interval.  In \cite{gollakota_testable} an algorithm for learning
any function of a \emph{constant number} of halfspaces is given with runtime $d^{\poly(1/\eps)}$ (in general, their run-time scales double exponentially in the number of halfspaces, see Theorem~5.6 in their paper).
Using \Cref{thm:intro-subgaussian} we are able to obtain efficient algorithms for agnostic learning under concentration and anti-concentration for functions of any number of halfspaces. We note that our runtime is only single exponential in the number of halfspaces, improving on the double-exponential dependence in \cite{gollakota_testable}. 

\begin{corollary}[Informal, see also \Cref{corollary: agnostic_subexp_functions_halfspaces}]
Let $D$ be a distribution on $\R^d \times \{\pm 1\}$ whose $\x$-marginal is sub-Gaussian
and anti-concentrated.
There exists an algorithm that agnostically learns arbitrary functions of $k$ halfspaces
 with $N = d^{\poly(\frac{k}{\epsilon})}\log(\frac{1}{\delta})$ samples and $\poly(d,N)$ runtime.
\end{corollary}

\subsection{Technical Overview}
\label{sec:techniques}

Our main plan is to use low-degree polynomials that can be efficiently optimized via $L_1$-regression, similar to the works of \cite{KKMS:05, KOS:08}. In general, in the agnostic setting, one has to construct a polynomial $p(\mathbf{x})$ that achieves almost optimal $L_1$ error with the label $y$. To do this, we have to prove that for every concept $f$ in the class, there exists a low-degree polynomial $p$ such that ${\E}_{\mathbf{x} \sim D_{\mathbf{x}}} [|p(\mathbf{x}) - f(\mathbf{x})|] \leq \epsilon$.

In the distribution-specific setting, i.e., when $\mathbf{x}$ comes from the Gaussian or the uniform on the hypercube, it is known that such a polynomial of degree $\text{poly}(\Gamma/\epsilon)$ exists \cite{KOS:08}. However, without assumptions on $D$, low-degree polynomial approximations of $f$ do not exist even when the $f$ is a simple concept such as a linear threshold function.

\paragraph{Polynomial Approximation in the Low-Dimensional Space}

Our high-level plan is to treat the smoothed learning setting as a non-worst-case approximation setting and show that given some $f$, with high probability over the smoothing $\mathbf{z}$, the translated concept $\mathbf{x} \mapsto f(\mathbf{x} + \sigma \mathbf{z})$ will have a low-degree polynomial approximation. For simplicity, in this sketch, we will assume that $\sigma = 1$. The general case can be found in the full proof; see \Cref{subsec:3.1} and also \Cref{rem:scaling-invariance}. We will construct a family of polynomials $p_{\mathbf{z}}(\mathbf{x})$ such that their expected $L_1$ error over the smoothing $\mathbf{z}$ is small:

\[
\E_{\mathbf{z} \sim \mathcal{N}}\Big[ \E_{\mathbf{x} \sim D_{\mathbf{x}}} [|p_{\mathbf{z}}(\mathbf{x}) - f(\mathbf{x} + \mathbf{z})|] \Big] \leq \epsilon\,.
\]

We observe that since every $f(\mathbf{x})$ depends only on a $k$-dimensional space $U$, the projection of the input $\mathbf{x}$ down to $U$ is just a linear transformation that does not affect the degree of polynomial approximation. Therefore, from now on, we may assume $\mathbf{x}$ lies in the $k$-dimensional space $U$ and construct our polynomial approximation there.

\paragraph{Duality Between Input and Smoothing Parameter}

Our first step is to think of the smoothing random variable as the actual input to the function and treat $\mathbf{x}$ as a fixed parameter. Therefore, as a function of $\mathbf{z}$, we now have to approximate the translated function $f_{\mathbf{x}}(\mathbf{z}) = f(\mathbf{x} + \mathbf{z})$. Even though $\mathbf{z}$ is not available to the learner, when we think of $f_{\mathbf{x}}(\mathbf{z})$ as a function of the Gaussian noise random variable, we can utilize strong approximation results known under the Gaussian. In particular, we can replace the boolean function $f_{\mathbf{x}}(\vec{z})$ by its smooth approximation given by the Ornstein-Uhlenbeck operator defined as $T_{\rho} f_{\mathbf{x}}(\mathbf{z})=\E_{\vec{s} \sim \mathcal{N}}[f_{\mathbf{x}}(\sqrt{1-\rho^2} \cdot \mathbf{z} + \rho ~ \vec{s} )]$.

Using the fact that the concept class of \Cref{def:bounded-surface-area-concepts} is closed under translation, we have that, since $\Gamma(f) \leq \Gamma$, the GSA of the translated concept $f_{\mathbf{x}}(\mathbf{z})$ as a function of $\mathbf{z}$ is also at most $\Gamma$. Using this fact and a result from Ledoux and Pisier (see \Cref{prop:ou_boolean}) that bounds the $L_1$ approximation error of the Ornstein-Uhlenbeck noise operator, we obtain that with $\rho = \text{poly}(\epsilon/\Gamma)$ it holds that 

\[
\E_{\mathbf{z} \sim \mathcal{N}}[|T_\rho f_{\mathbf{x}}(\mathbf{z}) - f_{\mathbf{x}}(\mathbf{z})|] \leq \epsilon \,.
\]

So far, we replaced $f_{\mathbf{x}}$ with $T_\rho f_{\mathbf{x}}$, but have we made progress? We observe that $T_\rho f_{\mathbf{x}}(\mathbf{z}) = \E_{\vec{s} \sim \mathcal{N}}[f(\mathbf{x} + \sqrt{1-\rho^2} \mathbf{z} + \rho \vec{s})]$. The variable $\mathbf{x}$, which is supposed to be input of the polynomial, is still in the function $f$. Without distributional assumptions on $D_{\x}$ the degree to approximate $f$ can be arbitrarily large.

\paragraph{From Approximating $f(\cdot)$ to Approximating Density Ratios}

To avoid approximating the concept $f$, we observe that we can express the Ornstein-Uhlenbeck operator as follows:

\begin{align*}
T_{\rho}f_{\mathbf{x}}(\mathbf{z}) &= \E_{\vec{s}\sim \mathcal{N}(\mathbf{x}/\rho,\mathbf{I})}\bigl[f(\sqrt{1-\rho^2}\mathbf{z}+\rho\vec{s})\bigr]  
= \E_{\vec{s}\sim Q}\biggl[f(\sqrt{1-\rho^2}\mathbf{z}+\rho\vec{s}) \cdot \frac{\mathcal{N}(\vec{s};\mathbf{x}/\rho,\mathbf{I})}{Q(\vec{s})}\biggr] \,,
\end{align*} 
where $Q(\vec{s})$ is a distribution that we carefully design. We have managed to decouple  the variable $\mathbf{x}$ from the function $f(\cdot)$, and now the task is to create a polynomial approximation of the density ratio $\frac{\mathcal{N}(\vec{s};\mathbf{x}/\rho,\mathbf{I})}{Q(\vec{s})}$, which --- at the very least --- is a continuous function of $\mathbf{x}$.\footnote{Note that the density ratio can be viewed as a Radon--Nikodym derivative between the two distributions. This step is closely related to importance sampling in statistics.}
For this to be possible, we need that the ratio of densities has a bounded $L_1$ norm with respect to $\mathbf{x} \sim D_{\mathbf{x}}$. When $\mathbf{x}$ is bounded, we can simply select $Q$ to be the standard Gaussian; see \Cref{lem: boolean_bounded_approx}. For sub-Gaussian (or strictly sub-exponential) marginals, we select a distribution $Q$ with heavier (exponential) tails than $D_{\mathbf{x}}$. For this overview, we focus on the case of bounded marginals and refer to \Cref{sec:polynomial-sub-exp} for the more general result.

We observe that the approximating function has to be polynomial in $\mathbf{x}$ but can be an arbitrary function of $\mathbf{z}$ and $\vec{s}$. Therefore, we select a weighted combination of polynomials (that is still a polynomial in $\mathbf{x}$ but not a polynomial in $\mathbf{z}$):
\[
p_{\mathbf{z}}(\mathbf{x}) = \E_{\vec{s}\sim Q}[f(\sqrt{1-\rho^2} \mathbf{z} + \rho \vec{s}) ~ q(\mathbf{x}, \vec{s})]
\,.
\]
To bound the $L_1$ distance of $\text{T}_\rho f_{\mathbf{x}}(\mathbf{z})$ and $p_{\mathbf{z}}(\mathbf{x})$, since $f$ is boolean and, in particular, bounded, it suffices to show that the polynomial $q(\mathbf{x}, \vec{s})$ approximates the ratio of normals $\mathcal{N}(\vec{s}; \mathbf{x}/\rho, \mathbf{I})/\mathcal{N}(\vec{s})$. We construct an explicit polynomial approximation of this ratio using the Taylor expansion of the exponential function and show that a degree roughly $\text{poly}(\log(1/\epsilon)/\rho)$ suffices; see \Cref{claim:approx_exp_pdf_taylor_main_body}. By our choice of $\rho$, we conclude that the degree of the family of polynomials $p_{\mathbf{z}}$ that we construct is at most $\text{poly}(\Gamma/\epsilon)$.

\paragraph{Dimension Reduction and Polynomial Regression}

Having constructed polynomial approximations with high probability over the smoothing random variable $\mathbf{z}$, we can use the standard $L_1$ polynomial regression algorithm; see \cite{KKMS:05,KOS:08}. For the case of bounded marginals, we show that we can also perform a dimension-reduction preprocessing step by a random projection. Even though the class of concepts of \Cref{def:bounded-surface-area-concepts} is non-parametric, we show that under bounded GSA, it is possible to reduce the dimension to $\text{poly}(k \Gamma/\epsilon)$; see \Cref{sec:efficient-algorithms}. 

\section{Preliminaries and Notation}
\paragraph{Notation}
We use small boldface characters for vectors and capital bold characters for matrices.
We use $[d]$ to denote the set $\{1, 2, \ldots, d\}$.
For a vector $\vec x \in \R^d$ and $i \in [d]$, $\vec x_i$ denotes the $i$-th coordinate of $\vec x$, and $\|\vec x\|_2 := \sqrt{ \sum_{i=1}^d \vec x_i^2 }$ the $\ell_2$ norm of $\vec x$.
We use $\vec x \cdot \vec y := \sum_{i=1}^n \vec x_i  \vec y_i$ as the inner product between them. We use $\mathds 1\{ E  \}$  to denote the indicator function of some event $E$.  We use $\E_{\vec x \sim D}[f(\vec x)]$ for the expectation of $f(\vec x)$ according to the distribution $D$ and $\pr_D[E]$ for the probability of event $E$ under $D$. For simplicity, we may omit the distribution when it is clear from the context. 
For $\vec \mu \in \R^d, \vec \Sigma \in \R^{d \times d}$, we denote by $\normal (\vec \mu, \vec \Sigma)$  the $d$-dimensional Gaussian distribution
with mean $\vec \mu$ and covariance $\vec \Sigma$.  We simply use $\mathcal N$
for the standard normal distribution.  In cases where the dimension is not clear from the context we shall use $\mathcal{N}_k$ to denote the standard normal on $k$-dimensions.
For $(\vec x, y)$ distributed according to $D$, we denote $D_{\vec x}$
to be the marginal distribution of $\vec x$.

\section{Smoothed Agnostic Learning under Concentration}
\label{sec:3}
In this section, we present our algorithms for smoothed learning under bounded 
and (strictly) sub-exponential marginals. The polynomial approximation results are in \Cref{subsec:3.1} for bounded marginals and in \Cref{sec:polynomial-sub-exp} for strictly sub-exponential marginals. In 
\Cref{sec:efficient-algorithms} we present our algorithmic results and the dimension-reduction process for learning under bounded-marginals.

\subsection{Polynomial Approximation: Bounded Marginals}
\label{subsec:3.1}

In this section we present and prove our main polynomial approximation
result for bounded marginals showing that, in expectation over the noise variable $\vec z$,
there exists some polynomial $p_{\vec z}(\x)$ that approximates the
translated concept function $f(\vec x + \vec z)$.  The proof of \Cref{lem: boolean_bounded_approx}
is split into two steps. 
Similar to our discussion in \Cref{sec:techniques}, we first fix $\x$ and try to approximate $f_{\x}(\z)$.
The first step is to replace $f$ by its smoothed version $T_\rho f_{\x}$ (see \Cref{definition:orstein-uhlenbeck}) and 
show that it is close to $f_{\x}$. 
The second step, see \Cref{lem:ou_approx_bounded_main_body}, 
is to construct a polynomial approximation of $T_\rho f_{\x}$ (similar to the way we constructed polynomial approximations to the Hermite coefficients of $f_\x$ in \Cref{sec:techniques}).

\begin{proposition}[Polynomial Approximation of Random Translations]
    \label{lem: boolean_bounded_approx} 
   Fix $\eps > 0$ and  sufficiently large universal constant $C > 0$. 
    Let $D$ be a distribution on $\R^d$ such that all points $\x$ in the support of $D$ have $\twonorm{\x}\leq R$. Let $f \in \mathcal F(k, \Gamma)$.
    There exists a family of polynomials $p_{\z}$ parameterized by $\z$ of degree at most $C(\Gamma/\epsilon)^4 R^2\log(1/\epsilon)$ such that
     $\E_{\z\sim \Gauss}\E_{\x \sim D}\bigl[|p_{\z}(\x)-f(\x+\z)|\bigr]$ is at most $\epsilon$, and every coefficient of $p_{\z}$ is bounded by $d^{C\bigl((\Gamma/\epsilon)^4 R^2\log(1/\epsilon)\bigr)^2}$. 
    \end{proposition}
\begin{remark}
\label{rem:scaling-invariance}
We remark that in \Cref{lem: boolean_bounded_approx} we have assumed that $\sigma = 1$ to simplify notation.  Using the fact that the surface area bound of the concepts of \Cref{def:bounded-surface-area-concepts} is invariant under translation and positive scaling, we can apply \Cref{lem: boolean_bounded_approx}
with $R' = R/\sigma$ for the function $\x \mapsto f(\sigma (\frac{\x}{\sigma} + \vec z))$ and obtain a polynomial of degree $\wt{O}( (\Gamma/\eps)^4 (R/\sigma)^2)$.  See also \Cref{thm:random_proj_bounded}.
\end{remark}
\begin{proof}
We use the following Gaussian noise operator to transform $f(\cdot)$ into a smooth function 
that is easier to approximate.
\begin{definition}[Ornstein-Uhlenbeck Noise Operator]\label{definition:orstein-uhlenbeck}
    Let $k\in \N$ and $\rho\in [0,1]$. 
    We define the Ornstein-Uhlenbeck operator $\ouoperator_{\rho}:\{\R^d \to\R\} \to \{\R^d \to\R\}$ that maps $f:\R^d\to\R$ to the function $\ouoperator_\rho f:\R^d \to \R$ with 
    \(
        \ouoperator_{\rho}f(\x)=\E_{\z\sim \Gauss}[f(\sqrt{1-\rho^2} \cdot \x+\rho\cdot  \z )]\,.
    \)
\end{definition}
    
    We will use the following result showing that under the assumption that 
    some function $g$ has bounded GSA, the Ornstein-Uhlenbeck operator $T_\rho g$ 
    yields a good approximation to $g$ in $L_1$.
    \begin{theorem}[\cite{Pisier:86, ledoux1994}]
    \label{prop:ou_boolean}
        Let $\rho\in [0,1]$ and consider a function $f:\R^d \to \{\pm 1\}$. It holds
       \[
       \E_\z\left[ \left| \ouoperator_\rho f(\z) - f(\z) \right| \right] \le 2\sqrt{\pi \rho} \cdot \Gamma(f).
       \]
    \end{theorem}
    We pause here to give some intuition about the above theorem. We note that for any boolean function, the left hand side is proportional to the probability that two $(1-\rho)$ correlated gaussians land on opposite sides of the decision boundary of $f$. Intuitively this should scale with the surface area as $\rho$ tends to zero (this probability is known as the noise sensitivity). The above theorem makes this quantitative. We note that it is precisely this theorem that allows low degree polynomial approximation over the gaussian distribution (see Theorem 15 in \cite{KOS:08}).

We now continue with the proof. Let $f_{\x}$ be the translated function defined as $f_{\x}(\z)=f(\x+\z)$.  
From \Cref{prop:ou_boolean}, we have 
$$\E_{\z\sim \mathcal{N}}\left[\left|T_\rho f_{\x}(\z)-f(\z+\x)\right|\right]\leq  2\sqrt{\pi\rho}\cdot \Gamma.$$ 

Choosing $\rho=O(\epsilon^2/\Gamma^2)$ makes this error at most $\epsilon/2$. 
    We now approximate $T_{\rho}f_{\x}$ using a polynomial.  
    To do this we prove the following result. We provide a proof sketch here, and refer to the appendix for the details and the formal statement, see \Cref{lem:ou_approx_bounded}.
\begin{lemma}[Approximating the Ornstein-Uhlenbeck Smoothed Concept $T_\rho f_{\vec x}(\cdot)$]
    \label{lem:ou_approx_bounded_main_body}
    Let $D$ be a distribution on $\R^d$ with every point $\x$ in the support of $D$ having $\twonorm{\x}$ at most $R$. Let $f:\R^d\to\{\pm 1\}$ and $f_{\x}:\R^d\to\R$ be defined as $f_{\x}(\z)=f(\x+\z)$. Then, for any $\epsilon>0$, there exist polynomials $p_{\z}$ parameterized by $\z$ 
    for degree at most  $O((R/\rho)^2\log(1/\epsilon))$, such that 
    \begin{enumerate}
        \item $\E_{\x \sim D}\E_{\z\sim \Gauss}\bigl[|p_{\z}(\x)-T_{\rho}f_{\x}(\z)|\bigr]\leq \epsilon$,
\item Every coefficient of $p_{\z}$ is bounded in absolute value by $d^{C\bigl((R/\rho)^2\log(1/\epsilon)\bigr)^2}$\,.
        
    \end{enumerate}
\end{lemma}
Before we prove \Cref{lem:ou_approx_bounded_main_body} we use it to conclude the proof of \Cref{lem: boolean_bounded_approx}.  From \Cref{lem:ou_approx_bounded_main_body}, we get a polynomial $p_{\z}$ of degree $C(\Gamma/\epsilon)^4 R^2\log(1/\epsilon)$ such that $\E_{\x\sim D}\E_{\z\sim \Gauss}\bigl[|T_{\rho}f_{\x}(\z)-p_{\z}(\x)|\bigr]\leq \epsilon/2$ where $C$ is a large universal constant.  
The coefficients of $p_{\z}$ are bounded by $d^{C\bigl((\Gamma/\epsilon)^4 R^2\log(1/\epsilon)\bigr)^2}$.    
By a triangle inequality, we get  $\E_{\x\sim D}\E_{\z\sim \Gauss_k}\bigl[|p_{\z}(\x)-f(\z+\x)|\bigr]\leq \epsilon$.
\paragraph{Sketch of the Proof of \Cref{lem:ou_approx_bounded_main_body}}
   We observe that  $T_{\rho}f_{\x}(\z)=\E_{\vec s\sim \Gauss}[f(\x+\sqrt{1-\rho^2}\z+\rho \vec s)]$ has 
   the variable $\x$ inside $f$. Recall that our goal is to construct a polynomial in $\x$ and, since we have no control over $f$ (which can possibly be very hard to approximate pointwise with a polynomial), we decouple $f$ and $\x$ in the expression of $T_{\rho}f_{\x}$ by writing the function as an expectation over a Gaussian centered at $\x/\rho$. 
    \begin{align*}
        T_{\rho}f_{\x}(\z)&=\E_{\vec s\sim \Gauss}\bigl[f(\x+\sqrt{1-\rho^2}\z+\rho\vec s)\bigr]
        =\E_{\vec s\sim \Gauss(\vec x/\rho,\vec I)}\bigl[f(\sqrt{1-\rho^2}\z+\rho\vec s)\bigr],
        \end{align*}
        Next, we can recenter the expectation around zero and express the Ornstein-Uhlenbeck operator as follows:
    \begin{align*}
        T_{\rho}f_{\x}(\z) &=\E_{\vec s\sim \Gauss}\biggl[f(\sqrt{1-\rho^2}\z+\rho\vec s)\cdot \frac{\Gauss(\vec s;\x/\rho,\vec I)}{N(\vec s;\vec 0,\vec I)}\biggr] =
        \E_{\vec s\sim \Gauss}\biggl[
        f(\sqrt{1-\rho^2}\z+\rho\vec s)\cdot e^{-\frac{\twonorm{\x}^2}{2\rho^2}+(\x/\rho)\cdot \vec s}
        \biggr] \,.
    \end{align*}
    To construct our polynomial, we now approximate ${e^{-\frac{\twonorm{\x}^2}{2\rho^2}+ (\x/\rho)\cdot\vec s}}$ using the 1-dimensional Taylor expansion of the exponential function $q(\x,\vec s)=q_m\bigl({-\frac{\twonorm{\x}^2}{2\rho^2}+(\x/\rho)\cdot\vec s}\bigr)$ where $q_{m}(t)=1+\sum_{i=1}^{m-1}\frac{t^i}{i!}$ is the degree $m-1$ Taylor approximation of $e^x$. Thus, our final polynomial $p_{\z}(\x)$ is 
    \begin{equation}\label{eq:polynomial-p-z}
        p_{\z}(\x)=\E_{\vec s\sim \Gauss}\bigl[f(\sqrt{1-\rho^2}\z+\rho\vec s)\cdot q(\x,\vec s)\bigr]\,.
        \end{equation}
    Let $\Delta(\x)$ be defined as the error term $\E_{\z \sim \normal}[|p_{\z}(\x)-T_{\rho}(f_{\x}(\z))|]$. 
    We have that  
    \begin{align}
    \Delta(\x) &=
    \E_{\z\sim \Gauss}\Big[\E_{\vec s\sim \Gauss}\bigl[|f(\sqrt{1-\rho^2}\z+\rho\vec s)|\cdot 
    \big|q(\x,\vec s)-e^{-\frac{\twonorm{\x}^2}{2\rho^2}+(\x/\rho)\cdot\vec s}\big| \Big]  \Big]
    \nonumber
    \\
    &\leq 
    \E_{\vec s\sim \Gauss}\Big[\big|q(\x,\vec s)-e^{-\frac{\twonorm{\x}^2}{2\rho^2}+(\x/\rho)\cdot\vec s}\big| \Big]
    \,,
    \label{eq:error-of-delta}
    \end{align}
    where for the inequality we used the fact that $|f(\vec x)|=1$ for all $\x$.
    We now observe that when $\vec s \sim \normal$ the random variable 
    $-\|\x\|_2^2/(2 \rho^2) + (\x/\rho) \cdot \vec s$ is distributed as
    $\normal(-\alpha^2/2, \alpha^2)$, where $\alpha = -\|\x\|_2^2/\rho^2$.  
    Therefore, we have reduced the original polynomial approximation problem to showing that 
    the Taylor expansion of the exponential function converges fast in $L_1$ to $e^x$ with 
    respect to $\normal(-\alpha^2/2, \alpha^2)$.  The proof of the following lemma is technical 
    and can be found in the appendix (see \Cref{lem:ou_approx_bounded}). Here we give
    a heuristic argument.
\begin{lemma}[Approximation of $e^x$ with respect to $\normal(-\alpha^2/2, \alpha^2)$]
\label{claim:approx_exp_pdf_taylor_main_body}
Fix $\alpha >0$ and sufficiently large universal constant $C>0$.
Let $p$ be the polynomial $p(x)=\sum_{i=0}^{m-1}\frac{x^{i}}{i!}$ with $m = C\alpha^2\log(1/\epsilon)$. 
We have that  $\E_{x\sim \normal(-\alpha^2/2,\alpha^2)}[|e^x-p(x)|]\leq \epsilon$.
\end{lemma}
\begin{subproof}[Sketch]
We first observe that since the Gaussian has mean $-\alpha^2/2$ and variance $\alpha^2$ using the 
strong concentration of the Gaussian (whose tail decays faster than the exponential growth of $e^x$ and
its Taylor expansion, see \Cref{lem:ou_approx_bounded} for more details) we may assume that 
we only have to approximate the  exponential function in the interval 
$[-\alpha^2/2 - O(\alpha\sqrt{\log(1/\eps)}), -\alpha^2/2 + O(\alpha\sqrt{\log(1/\eps)})]$.
By Taylor's theorem we have that for any interval $[a,b]$ it holds that 
$|p(x) - e^x| \leq e^{b} \max(|a|, |b|)^m/m!$.  Therefore, we have that by picking degree 
$m = O(\alpha^2 \log(1/\eps))$ we can make the error of the Taylor expansion at most $\eps$.
\end{subproof}

Using \Cref{claim:approx_exp_pdf_taylor_main_body} with $\alpha=\twonorm{\x/\rho}$ we obtain that 
with degree $O((R/\rho)^2 \log(1/\eps))$ the $L_1$ error of the polynomial $q(\x, \vec s)$ in \Cref{eq:error-of-delta} is at most $\eps$.   To bound the coefficients of the polynomial $p_\vec z(\x)$ we use the fact that $f(\x)$ is boolean (and therefore bounded) and the fact that the input of the Taylor expansion in $q(\x, \vec s)$ is bounded.
For 
the full proof, see the appendix.
\end{proof}

\subsection{Polynomial Approximation: Strictly Sub-Exponential Marginals}
\label{sec:polynomial-sub-exp}
In this section we prove our polynomial approximation for the more general class of Strictly Sub-Exponential
distributions, defined as follows.
\begin{definition}[Strictly Sub-exponential Distributions]
    \label{def:strict_subexp}
    A distribution $D$ on $\R^d$ is $(\alpha,\lambda)$-strictly sub-exponential for $\alpha,\lambda>0$ 
    if for all $\twonorm{\vv}=1$, 
        \(
          \pr_{\x\sim D}[|\x\cdot \vv|>t ]\leq 2\cdot e^{-(t/\lambda)^{1+\alpha}} 
        \).
\end{definition}
Our main goal in this section is to prove the following polynomial approximation result which is a generalization of \Cref{lem: boolean_bounded_approx}. We refer to \Cref{lem:poly_approx_boolean_subexp} 
in the appendix for the formal statement.
\begin{proposition}[Polynomial Approximation: Strictly Sub-Exponential Marginals]
    \label{lem:poly_approx_boolean_subexp_main_body}
    Let $C$ be a large universal constant.
    Let $D$ be a distribution on $\R^k$ that is $(\alpha,\lambda)$-strictly subexponential. Let 
    $f:\R^k \mapsto \{\pm 1\}$ be a boolean function in $\mathcal{F}(k, \Gamma)$.
    Then there exist polynomials $p_{\z}$ of degree at most 
     $\left(C \lambda k\Gamma^2\log (1/\epsilon)/\epsilon^2\right)^{64(1+1/\alpha)^3}$,
    parameterized by $\z$ whose (expected) $L_1$ error is $\E_{\z\sim \Gauss_k}\E_{\vu\sim D}\bigl[|p_{\z}(\vu)-f(\z+\vu)|\bigr] \leq \epsilon$.
\end{proposition}
The main proof idea is similar to that of \Cref{lem: boolean_bounded_approx}.   However, there are significantly more technical hurdles in constructing the approximating polynomial for this case and
we will only highlight some of the main differences and refer to the appendix for the 
full proof. Similar to the proof of \Cref{lem: boolean_bounded_approx}, by using the result of 
Ledoux and Pisier \Cref{prop:ou_boolean} we obtain that it suffices to approximate the function 
$T_\rho f_\x(\z)$ with some polynomial $p_\z(\x)$.  Since $f$ is low-dimensional (see \Cref{def:bounded-surface-area-concepts}) we can write $f(\x) = f(\vec U^T \vec U \x)$ for some $k \times d$ 
projection matrix $\vec U$. Since the polynomial regression algorithm is able to learn this linear
transformation, from now on we assume that $f$ is an explicit $k$ dimensional function 
$f(\vu):\R^k \mapsto \{\pm 1\}$.  We will show that there exists a polynomial of degree 
at most $\left(C\lambda k\log (1/\epsilon)/\rho\right)^{64(1+1/\alpha)^3}$ that approximates 
$T_\rho f_\vu(\z)$. Similar to the proof of \Cref{lem: boolean_bounded_approx}, 
the first step is to re-write the expression of $T_\rho f_\vu(\z)$ so that $\vu$ does not appear inside the target function $f$.  We observe that for any distribution $Q$ we have
    \begin{align*}
        T_{\rho}f_{\vu}(\z) &=\E_{\vec s\sim Q}\biggl[f(\sqrt{1-\rho^2}\z+\rho\vec s)\cdot \frac{\Gauss(\vec s;\vu/\rho,I)}{Q(\vec s)}\biggr]\\
        &=e^{-\frac{\twonorm{\vu}^2}{2\rho^2}}\E_{\vec s\sim Q}\biggl[f(\sqrt{1-\rho^2}\z+\rho\vec s)\cdot e^{-\frac{\twonorm{\vec s}^2}{2}-\log Q(\vec s)}e^{ (\vu/\rho)\cdot\vec s}\biggr] \,.
    \end{align*}
    We observe that we can no longer take $Q$ to be a Gaussian (like we did in \Cref{lem: boolean_bounded_approx}) because when $\vec u$ has weaker tails than the normal 
    density the 
    $\E_{\vec u\sim D}[\big(\frac{\Gauss(\vec s;\vu/\rho,I)}{Q(\vec s)} \big)^2] = +\infty$.
    To avoid this we take $Q$ to be the distribution on $\R^k$ with probability distribution 
    function $Q(\vec s)\propto e^{-\|\vec s\|_1}$ which has exponential tails.  We show, see \Cref{lem:ratio_subexp} in the appendix, that
   $ \E_{\x\sim Q}[(\frac{\mathcal{N}(\x; \vu, \vec I)}{Q(\x)})^2] \leq  C^ke^{C\norm{\vu}_1}\,. $
   Beyond working with the exponential reweighting function, another technical complication is that 
   we now have to carefully create a polynomial approximation over a strictly sub-exponential 
   distribution for the function $e^{-\|\vec s\|_2^2}$, see \Cref{thm:exp_approx_subexp} in appendix.  To do this we use a tighter polynomial approximation using Chebyshev polynomials.
\begin{remark}
\label{rem:tail_assumption}
    We note that the above proof technique fails when we consider the case of sub-exponential distributions. To see this, first note that our choice of $Q(\vec{s})\propto e^{-\norm{\vec{s}}_1}$ no longer works as we again have the same issue of $\E_{\vec u\sim D}[\big(\frac{\Gauss(\vec s;\vu/\rho,I)}{Q(\vec s)} \big)^2]$ tending to infinity. Another reason the proof fails to extend is that we do not know how to construct polynomial approximators for $e^{-\twonorm{\vec{s}}^2}$ over sub-exponential distributions. This is a crucial step in our proof and it seems to require the tail assumptions we make.
\end{remark}
\subsection{Efficient Algorithms for Learning under Concentration}
\label{sec:efficient-algorithms}

Given the polynomial approximation construction of the previous sections one can directly
run $L_1$ polynomial regression to minimize $\E_{(\x, y)\sim D}[ |p(\x) - y|]$ similar to \cite{KKMS:05}.
We now state our main theorem for strictly sub-exponential distributions.
\begin{theorem}
\label{thm:smooth_learning_subgaussian-main}
    Let $k \in \mathbb Z_+$ and $\epsilon, \delta,\sigma \in (0,1)$.  Let $D$ be a distribution on $\R^{d}\times \{\pm 1\}$ such that the marginal distribution is $(\alpha,\lambda)$-strictly subexponential. There exists an algorithm that draws $N = d^{\poly((k\lambda\Gamma/(\sigma\epsilon))^{(1+1/\alpha)^3}))}$ samples, runs in time $\poly(d,N)$, and computes a hypothesis $h(\x)$ such that, with probability at least $1-\delta$, it holds  
    $ \pr_{(\x, y) \sim \Djoint}[y\neq h(\x)] \le \opt_\sigma + \eps  $.
\end{theorem}

In the case of bounded marginals, we can significantly reduce the runtime of the algorithm
by performing a dimension reduction via a random Gaussian projection similar to the
works of \cite{ArriagaVempala:99} and \cite{klivans-servedio-margin}.   We show that 
when the $\x$-marginal of the distribution is bounded then we can perform a random projection
to reduce dimension down to $\poly(k \Gamma/\eps)$ for the class of concepts of \Cref{def:bounded-surface-area-concepts}.  Assuming that $f\in \mathcal F(k, \Gamma)$ we have that there exists a $k\times d$ matrix 
$\vec U$ such that $f(\x) = f(\vec U^T \vec U \x)$.  Let $\vec A$ be the random projection matrix.
It suffices to show that concept $f(\vec A \vec x)$ is close in $L_1$ to the original concept $f(\x)$. 
We once again use the fact that we can exchange the order of expectation so that we are able to use
the properties of the random Gaussian smoothing.  We show, see \Cref{lem:gsa_shift} in the appendix, that for every $f \in \mathcal{F}(k, \Gamma)$ it holds that 
    $\E_{\z\sim \Gauss}\bigl[|f(\vu+\z)-f(\vv+\z)|\bigr]\leq O(\Gamma\cdot\twonorm{\vu-\vv})$.
    Therefore, we obtain that a random projection down to $\poly(k\Gamma/\epsilon)$ dimensions will imply
    that $\E_{\x \sim \D_\x}\E_{\vec z \sim \normal}[[f(\vec A \vec x+ \vec z) - f(\vec \x + \vec \z)|] \leq \eps$. By performing polynomial regression in the low-dimensional space we obtain the following improved
    runtime for bounded $\x$-marginals.
\begin{theorem}
\label{thm:random_proj_bounded_main}
    Let $k \in \mathbb Z_+$ and $\epsilon, \delta,\sigma \in (0,1)$. Let $D$ be a distribution on $\R^d\times \{\pm 1\}$ whose $\x$-marginal is bounded in the unit ball.  There is an algorithm that draws $N = k^{\tilde{O}\big((\Gamma/\epsilon)^4(1/\sigma^2)\big)}\log (\frac{1}{\delta})$ samples, runs in time $\poly(d,N)$, and computes a hypothesis $h(\x)$ such that, with probability at least $1-\delta$, it holds
    $ \pr_{(\x, y) \sim \Djoint}[y\neq h(\x)] \le \opt_\sigma + \eps\,$.
\end{theorem}

\section{Applications and Connections with Other Models}
In this section, we show connections between our model of smoothed learning and three important models that have been previously studied: (1) learning with margin, (2) learning under smoothed distributions and (3) learning with concentration and anti-concentration.  We briefly discuss (1) and (2) and defer (3) and other details to the appendix, see \Cref{app:applications}. 
\paragraph{Learning with Margin}
We show that any algorithm
for smoothed agnostic learning can be directly used to learn in the agnostic setting with margin.
For the formal definition of agnostic learning with $\gamma$-margin we refer to \Cref{eq:margin-opt}
and \Cref{definition:agnostic-margin}.
We denote by $\partial_\gamma f$ all points that are in distance at most $\gamma$ from the decision boundary.
We observe (see \Cref{lem:margin_smooth}) that $\opt_\sigma$ is not much larger than $\text{margin-opt}_\gamma$, as long we have that for any $\x\notin\partial_\gamma$ it holds that the value of $f$ is unlikely to change by the random perturbation:
\[
\opt_\sigma \leq \text{margin-opt}_\gamma + \sup_{\x \notin \partial_\gamma f} \pr_{\vec z \sim \normal}[f(\x + \sigma \vec z) \neq f(\x)] \,.
\]
For any boolean concept $f$, we show  (see \Cref{lemma:margin_general}) that 
as long as $\sigma$ is smaller than $\frac{\gamma}{\sqrt{k \log({1}/{\eps})}}$ it holds that 
$\sup_{\x \notin \partial_\gamma f} \pr_{\vec z \sim \normal}[f(\x + \sigma \vec z) \neq f(\x)] \leq \eps$.
While this holds in full generality, for specific concept classes we are able to provide better bounds.
In particular, for intersections of $k$ halfspaces we show, see \Cref{proposition:margin-invariance},
that picking $\sigma = \gamma/\sqrt{\log k/\eps}$ suffices.  Therefore, using \Cref{thm:random_proj_bounded}
we readily obtain the agnostic learning result for intersections of $k$-halfspaces of \Cref{inf:margin-intersections}.

\paragraph{Agnostic Learning with Distributional Assumptions}
As mentioned, our smoothed agnostic model generalizes agnostic learning with distributional assumptions.
We denote by $\opt$ the standard optimal agnostic error under a distribution $D$.  
We see (see \Cref{obs:agnostic_to_smooth} in the appendix) that 
\(
\opt_\sigma \leq \opt + \pr_{\x \sim D_\x, \vec z \sim \normal}[f(\x + \sigma \vec z) \neq f(\x)]\,.
\)
For the case of distribution smoothing we have that the smoothed distribution $D_\tau$  is the convolution of  $D_\x + \normal(\vec 0, \tau^2 \vec I)$.  In that case we have that 
$
\E_{\x \sim D_\x, \vec z \sim \normal }[f(\x + \tau \vec z_1 + \sigma \vec z_2) \neq f(\x+ \tau \vec z_1)] \leq 
O(\frac{\sigma \Gamma \sqrt{k}}{\tau}) \,.
$
Therefore, by choosing $\sigma = O(\eps \tau/(\Gamma \sqrt{k})$, we obtain that the gap between $\opt_{\sigma}$ and $\opt$
is at most $\eps$. For this value of $\sigma$, we are able to recover the strong results of \Cref{intro-cor:smoothed-distributions} which yields an exponential improvement over the prior work \cite{Kane2013LearningHalfspacesLogConcave}.

\section{Conclusion and Open Problems}
In this work we introduce a new beyond worst-case model for supervised learning and show that it is possible
to obtain efficient algorithms with runtime that were previously known only under very strong distributional assumptions, e.g., Gaussianity.  Moreover, we show that our framework and results generalize over several
settings considered in the literature --- often improving the best known results significantly (e.g., 
for the fundamental problem of learning intersections of $k$ halfspaces with margin). We now list three interesting open directions. 

First, can the runtime 
of \Cref{thm:intro-subgaussian} be improved: can we remove or make milder the exponential dependency on the intrinsic dimension $k$? Recall that there is no exponential dependence on the intrinsic dimension in the bounded marginal case. 

Second, what are the weakest assumptions on the $\x$-marginal that enable efficient learnability? There has been substantial recent progress in this direction. Subsequent work by Koehler and Wu \cite{koehler2025constructive} shows that efficient smoothed learnability is achievable even in the case of sub-exponential marginals (recall that this was a barrier for our techniques \Cref{rem:tail_assumption}). In fact, they show that efficient learnability under our smoothed analysis definition extends to all marginal distributions satisfying Carleman's conditions.

Third, are there other noise process that yield efficient learnability? Crucial to our analysis was the use of Gaussian perturbations. It is interesting to see if similar techniques could extend to other noise processes. In this direction, Meka and Kou \cite{kou2025smoothed} show that perturbations using the Boolean noise operator suffice for achieving efficient learnability of halfspaces over strictly-subexponential discrete distributions.


\printbibliography

\newpage
\appendix

\section{Gaussian Surface Area}

Here we give the formal definition of Gaussian Surface Area of a concept and present some known
bounds for the concept classes that we consider in this work.

\begin{definition}[Gaussian Surface Area]
\label{def:GSA}
Let $f$ be a boolean function, $\Gamma(f)$ is the Gaussian surface area of the set $A_f = \{\vu\in \R^k: f(\vu) = 1\}$, i.e., $\Gamma(f)$ is the following quantity.
\[
    \Gamma(f) = \liminf_{\delta \to 0} \frac{1}{\delta} \pr_{\z\sim\Gauss(0,I_k)}\bigr[\z \in A_f^\delta \setminus A_f\bigr]\,, \text{ where }A_f^\delta = \{\vu: \min_{\vv\in A_f}\twonorm{\vu-\vv} \le \delta\}.
\]
\end{definition}

We use the following bounds on Gaussian surface area in our results.
\begin{lemma}[Bounds on Gaussian surface area of various functions]
\label{lem:gsa_bounds}
The following are bounds on the Gaussian surface area of some common classes of functions:
\begin{enumerate}
    \item If $f$ is a halfspace, then $\Gamma(f)\leq O(1)$\cite{KOS:08},
    \item If $f$ is an intersection of $k$ halfspaces, then $\Gamma(f)\leq O(\sqrt{\log k})$(\cite{KOS:08}, due to Nazarov),
    \item If $f$ is an arbitrary boolean function of $k$ halfspaces, then $\Gamma(f)\leq O(k)$
    (folklore, see also \Cref{claim:gsa_func}),
    \item If $f$ is the degree $\ell$ polynomial threshold function(PTF), then $\Gamma(f)\leq O(\ell)$\cite{Kane11},
    \item if $f$ is an arbitrary convex set on $k$ variables, then $\Gamma(f)\leq O(k^{1/4})$\cite{Ball:93}.
\end{enumerate}
\end{lemma}

\begin{lemma}[Gaussian Surface Area of functions of $k$ halfspaces]
\label{claim:gsa_func}
    Let $f$ be a boolean function on $k$ halfspaces. Then, we have that $\Gamma(f)\leq O(k)$.
\end{lemma}
\begin{proof}
    Since $f$ is a boolean function on $k$ halfspaces, we have that for any input $\x$, $f(\x)=g\bigl(h_1(\x),h_2(\x),\ldots,h_k(\x)\bigr)$ where $h_1,\ldots,h_k$ are halfspaces on $\R^d$ and $g$ is an arbitrary boolean function. For any function $h$, let $A_h = \{\x\in\R^d: h(\x)=1\}$. For any set $S$, let $S^{\delta}$ denote the set of points at distance at most $\delta$ from $S$. For a boolean function $h$, let $h^c$ denote its complement.
    We observe that any point that is $\delta$-close to $A_h$ must be $\delta$-close to one of the defining halfspaces of $f$. Formally, we have that  $A_f^{\delta}\setminus A_f\subseteq \bigl(\bigcup_{i=1}^{k}(A_{h_i}^\delta\setminus A_{h_i})\bigr)\cup \bigl(\bigcup_{i=1}^{k}(A_{{h_i^c}}^\delta\setminus A_{h_i^c})\bigr)$. Then, we have that \[\Gamma(f) = \liminf_{\delta \to 0} \frac{1}{\delta} \pr_{\z\sim\Gauss(0,I_k)}\bigr[\z \in A_f^\delta \setminus A_f\bigr]\leq \sum_{i=1}^{k}\bigl(\Gamma(h_i)+\Gamma(h_i^c)\bigr)\leq O(k)\]
    where the first inequality comes from a union bound and the second comes from the bound on surface area of a single halfspace(\Cref{lem:gsa_bounds}). 
\end{proof}

\begin{table}[ht]
\centering
\begin{tabular}{|c|c|c|}
\hline
Concept Class & Bounded & Sub-Gaussian \\ 
\hline
\hline
Intersections  of $k$ halfspaces  & $\poly(d) \cdot k^{\poly(\frac{\log k}{\epsilon\sigma})}$ 
& $ d^{\poly(\frac{k}{\sigma \epsilon})} $
\\ 
\hline
Arbitrary functions of $k$ halfspaces & $\poly(d)\cdot k^{\poly(\frac{k}{\gamma\epsilon})}$ & 
$d^{\poly(\frac{k}{\sigma\epsilon})}$  \\

\hline
$k$-dimensional, $\ell$-degree PTFs  & $\poly(d)\cdot k^{\poly(\frac{k \ell}{\gamma\epsilon} ) }$ &  
$d^{\poly(\frac{k \ell}{\sigma\epsilon})}$  \\
\hline

$k$-dimensional convex sets  & $\poly(d)\cdot k^{\poly(\frac{k}{\gamma\epsilon} ) }$ &  
$d^{\poly(\frac{k}{\sigma\epsilon})}$  \\

\hline
\end{tabular}
\caption{Our results on smoothed agnostic learning.}
\label{tab:smoothed}
\end{table}

\section{Applications and Connections with Other Models}
\label{app:applications}
In this section, we show connections between our model of smoothed learning and three important models that have been previously studied: (1) learning with margin, (2) learning under smoothed distributions and (3) learning with concentration and anticoncentration. We prove that our results straightforwardly imply improved results in each of these models. For example, we give the first quasi-polynomial algorithm for agnostically learning intersections of halfspaces with margin. 
\subsection{Notation}
We introduce some notation for that we use in this section. For distribution $D$ on $\R^{d}\times \{\pm 1\}$ and function $f:\R^{d}\to \{\pm 1\}$, we denote the $0$-$1$ error as $\err(f,D)=\E_{(\x,y)\sim D}[f(\x)\neq y]$ and the smoothed error as $\err_{\sigma}(f,D)=\E_{\z\sim \Gauss}E_{(\x,y)\sim D}[f(\x+\sigma\z)\neq y]$. 
\subsection{Learning with Margin}
\label{sec:margin}

In this section we investigate the connection of our smoothed learning model with (agnostic) learning with margin. We first define the model. As discussed in the introduction, this model disincentivizes the adversary from placing points close to the boundary of the function in a bid to create worst case instances.
\begin{definition}[Agnostic Learning with Margin]\label{definition:agnostic-margin} 
Fix $\eps, \gamma >0$ and $\delta\in (0,1)$.
Let $\F\subseteq\{\R^d\to {\pm 1}\}$ be a class of Boolean concepts and let $\Dclass$ be a class of distributions over $\R^d$.
    Consider $\Djoint$ to be a distribution over $\R^d\times \{\pm 1\}$ such that $\Dmarginal\in \Dclass$.
    We say that the algorithm $\Alg$ learns $\F$ in the $\gamma$-margin setting if, after receiving i.i.d. samples from $\Djoint$, 
    $\Alg$ outputs a hypothesis $h:\R^d \to \{\pm 1\}$ such that,
    with probability at least $1-\delta$, over the samples it holds
    \begin{equation}
        \pr_{(\x, y) \sim \Djoint}[y\neq h(\x)] \le 
        \inf_{f \in \mathcal F}
         \E_{(\x, y) \sim \Djoint}\Big[ \sup_{\|\vec u\|_2 \leq \gamma} \1\{f(\x + \vec u) \neq y \} \Big] + \eps \,.
    \end{equation}
\end{definition}
\begin{remark}[Other definitions of agnostic learning with margin]
    We now highlight connections to other previously studied notions of margin.
    An equivalent model to \Cref{definition:agnostic-margin} is to define the $\gamma$-margin optimal error $\inf_{f \mathcal F} \E_{(\x,y)\sim D}[\1\{f(\x)\neq y\text{ or } \x\in \partial_{\gamma}f\}]$, see, e.g., \cite{diakonikolas19marginopt}.
    Another related model is defined in \cite{klivans-servedio-margin} where they define the set $\mathcal F_\gamma$ containing all functions $f\in \F$ that have $\gamma$-margin with respect to $D_\x$, i.e., $f \in \mathcal{F}_\gamma$ if $f \in \mathcal F$ and $\pr_{\x\sim D_{\x}}[x\in \partial_{\gamma}f]=0$ and define the (margin) optimal error as
    $\inf_{f \in \mathcal F_\gamma} \pr_{(\x, y) \sim D}[f(\x) \neq y]$.
    We remark that our result readily applies to this variant as well. 
\end{remark}

A key tool that we use in our reductions is the notion of \textit{Gaussian sensitivity} of a function at a point, which we now define. 
\begin{definition}[Gaussian Sensitivity at $\x$]
Let $f:\R^d \mapsto \{\pm 1\}^d$ be a boolean function.  We define the Gaussian $\sigma$-sensitivity of $f$ at $\x$ as $\mathbb{S}(\x;\sigma, f) \eqdef \pr_{\vec z \sim \normal}[f(\x + \sigma \z) \neq f(\x)]$.
\end{definition}

We now prove our reduction. We show that a learner for the smooth agnostic model of \Cref{definition:concept-smoothing} can readily give an algorithm for agnostic learning with margin if we have bounds on the Gaussian sensitivity at all points not in a $\gamma$ neighbourhood of the function's surface.

\begin{lemma}[From Margin to Smooth Agnostic]
\label{lem:margin_smooth}
Fix some boolean function $f:\R^d\mapsto \{\pm 1\}^d$ and some distribution $D$ over labeled
examples on $\R^d \times \{ \pm 1\}$.
We say that a point $\x$ lies in the $\gamma$-boundary of $f$, $\x \in \partial_\gamma f, $ 
if there exists $\vec u$ with $\|\vec u \|_2 \leq \gamma$, such that
$f(\x + \vec u) \neq f(\x)$. 
It holds
\[
\err_{\sigma}(f,D)
         \leq 
        \E_{(\x, y) \sim \Djoint}\Big[ \sup_{\|\vec u\|_2 \leq \gamma} \1\{f(\x + \vec u) \neq y \} \Big] 
         + \sup_{\x \notin \partial_\gamma f}
         \mathbb S(\x;\sigma, f)\,.
\]
\end{lemma}
\begin{proof}
We first observe that we change the order of expectations in the smoothed error of $f$ 
and consider 
        $
        \E_{(\x,y) \sim \Djoint}
        \E_{\z \sim \normal}
         \Big[\1\{f(\x + \sigma \vec z) \neq y \} \Big] 
         $.
         Now it suffices to show that for every $(\x, y)$ it holds
         \begin{equation}
         \label{eqn:eq:smooth_vs_margin}
         \pr_{\z \sim \normal}[f(\x+ \sigma \vec z) \neq y]
         \leq 
         \sup_{\|\vec u\|_2 \leq \gamma} \1\{f(\x + \vec u) \neq y\}
         + \sup_{\x \notin \partial_\gamma f} \mathbb S(\x; \sigma, f) \,.
         \end{equation}

         We now split into two cases, based on if the data point is in the $\gamma$-margin of $f$. First, consider a labeled example $(\x, y)$ where $\x \notin \partial_\gamma f$.   We have that 
         $\sup_{\|\vec u\|_2 \leq \gamma} \1\{f(\x+\vec u) \neq y\} = 
         \1\{f(\x) \neq y\} $. By the triangle inequality we have 
         $\1\{ f(\x + \sigma \vec z) \neq y \} 
         \leq  \1\{ f(\x) \neq y \}+ \1\{f(\x + \sigma \vec z) \neq f(\x)\}$. 

         We now consider the case when $\x\in \partial_\gamma f$.
         We 
         have that for any $y \in \{\pm 1\}$ there exists a $\vec u$ with $\|\vec u\|_2 \leq \gamma$ such that $f(\x + \vec u) \neq y$: if $y \neq f(\x)$ then pick $\vec u = \vec 0$ and if $y = f(\x)$ there exists $\vec u$ with $\|\vec u\|_2 \leq \gamma$ such that $f(\x + \vec u) \neq f(\x)$ and therefore $f(\x + \vec u) \neq y$.  Thus, we have $\1\{f(\x+\sigma \z)\}\leq \sup_{\twonorm{\vu}\leq \gamma}\1\{f(\x+\vu)\neq y\}$.
         
        \Cref{eqn:eq:smooth_vs_margin} now follows by taking expectation with respect to $\vec z\sim \Gauss$.
\end{proof}

Before, we can use the above lemma to get results in the margin setting, we need to bound the Gaussian sensitivity(at points $\gamma$ distance away from the surface) of various concepts. First, we bound this quantity for an arbitrary function. We then obtain stronger bounds for the class of intersections of halfspaces.

\begin{lemma}[Gaussian Sensitivity of arbitrary functions under $\gamma$-margin]
\label{lemma:margin_general}
Let $f:\R^k \mapsto \{\pm 1\}$ be a Boolean function.  
It holds 
$\sup_{\x \notin \partial_\gamma f}\mathbb S(\x; \sigma, f) \leq 
e^{- (\gamma/\sigma)^2/5 + k} $.  Equivalently for $\sigma = \gamma/\sqrt{5 k} + 1/\sqrt{\log(1/\eps)} $, it holds  
$\sup_{\x \notin \partial_\gamma f}\mathbb S(\x; \sigma, f)  \leq \eps$ \,.
\end{lemma}
\begin{proof}
We first observe that, since $\x \notin \partial_\gamma f$, the sign of $f$ will not change
as long as the perturbation $\sigma \vec z$ has norm smaller than $\gamma$. Thus,  we can bound
$\mathbb S(\x;\sigma, f)$ above by $\pr_{\z \sim \normal}[\sigma \|\vec z\|_2 \geq \gamma]$.
By using a standard tail bound for the $\chi^2$ distribution, see, e.g., Lemma 1 \cite{laurent2000adaptive}, we obtain that
$\pr_{\z \sim \normal}[\sigma \|\vec z\|_2 \geq \gamma] 
\leq \exp(-(\gamma/\sigma)^2/5 + k) $.
\end{proof}
For the sensitivity in the above bound to be less than $\epsilon$, we need $\sigma$ to be less than $\gamma/\sqrt{k}+\gamma/\sqrt{\log(1/\epsilon)}$. Recall that our smoothed learner's runtime scales exponentially in $1/\sigma$. Thus, we pay $\poly(k)$ in the exponent for arbitrary functions.
We now prove the improved bound for intersections of halfspaces. Note for any $\epsilon>0$, we have that for all $\sigma< \gamma/\sqrt{2\log(k/\epsilon)}$, the sensitivity at points at least $\gamma$ distance from the surface is bounded by $\epsilon$. This is a major improvement over the case of general functions and this in turn implies our quasi-polynomial time algorithm for agnostically learning intersections of halfspaces with margin(as $1/\sigma$ is now only $\poly(\log k))$.
\begin{lemma}[Gaussian Sensitivity of Intersections of $k$-halfspaces with $\gamma$-margin] \label{proposition:margin-invariance}
Let $f:\R^d \mapsto \{\pm 1\}$ be an intersection of $k$-halfspaces
and $D$ a distribution over $\R^d$.  It holds that
\[
\sup_{\x \notin \partial_\gamma f} \mathbb S(\x; \sigma, f) \leq 
k \exp(- \gamma^2/(2 \sigma^2)) \,.
\]

\end{lemma}
\begin{proof}
We split the proof in two cases.  We first assume that the point $\x$ lies inside the intersection of halfspaces, i.e., $f(\x) = +1$.  Denote by $l_i(\x)= \vec w_i \cdot \x + t_i$ the linear functions
defining the intersection of halfspaces, i.e., $f(\x) = 1$ if and only if $l_i(\x) \geq 0$ for all $i=1,\ldots, k$.  The probability that $\x + \sigma \vec z$ is classified differently (i.e., 
$f(\x + \sigma \vec z) = -1$) is then 
\begin{align}
\pr_{\z \sim \normal}[f(\x + \sigma \z) \neq f(\x)] 
= 
\pr_{\z \sim \normal}\Big[\bigcup_{i=1}^k \{ l_i(\vec x + \sigma \vec z) < 0 \}\Big] 
\leq 
\sum_{i=1}^k  \pr_{\z \sim \normal}\Big[ l_i(\vec x) + \sigma \vec z \cdot \vec w_i < 0\Big] 
\label{eq:union-bound-inside-margin-intersection}
\,.
\end{align}
For an $\x$ inside the intersection, it must be that its distance from all faces of the polytope 
is at least $\gamma$ as otherwise the $\gamma$-margin assumption would be violated.
Therefore, we have that for all $i=1, \ldots, k$ it holds that 
$l_i(\x) \geq \gamma \|\vec w_i\|_2$.  Therefore, we have that
$ 
\pr_{\z \sim \normal}[ l_i(\vec x) + \sigma \vec z \cdot \vec w_i < 0] 
\leq 
\pr_{\z \sim \normal}[ \|\vec w_i\|_2 \gamma + \sigma \vec z \cdot \vec w_i < 0] 
=
\pr_{t \sim \normal}[ t < -\gamma/\sigma] 
\leq e^{-\gamma^2/(2 \sigma^2)}
\,,
$
where we used the tail bound for the 1-dimension normal density.
Therefore, when $f(\x) = +1$ using \Cref{eq:union-bound-inside-margin-intersection} we 
conclude that 
\[
\pr_{\z \sim \normal}[f(\x + \sigma \z) \neq f(\x)] 
\leq k \exp(-\gamma^2/(2\sigma^2)) \,.
\]
We now move on to the case where $f(\x) = -1$.  Let $C = \{\x : f(\x) = +1\}$ be the convex set corresponding to the intersection, in this case we have 
\[
\pr_{\z \sim \normal}[f(\x + \sigma \z) \neq f(\x)] 
=
\pr_{\z \sim \normal}[\x + \sigma \z \in C] 
=
\pr_{\z \sim \normal(\x, \sigma^2 I)}[\z \in C] 
= \int \frac{1}{(2 \pi \sigma^2)^{d/2}} e^{-\frac{\|\x - \vec z\|_2^2}{ 2 \sigma^2}}
\1\{ \vec z \in C\}
\d \vec z
\,.
\]
Let $\pi_C(\x)$ be the projection of $\x$ (that lies outside of $C$) onto $C$. 
Since $C$ is convex, for any $\vec z \in C$, we have that 
$\|\x - \pi_C(\x)\|_2^2 + \|\pi_C(\x) - \vec z\|_2^2 \leq \|\vec x - \vec z\|_2^2$.
Therefore, it holds 
\begin{align*}
 \int \frac{1}{(2 \pi \sigma^2)^{d/2}} e^{-\frac{\|\x - \vec z\|_2^2}{ 2 \sigma^2}} \d \vec z
 &\leq 
 e^{- \frac{\|\vec x - \pi_C(\x)\|_2^2}{2 \sigma^2}} 
 \int \frac{1}{(2 \pi \sigma^2)^{d/2}} e^{-\frac{\|\pi_C(\x) - \vec z\|_2^2}{ 2 \sigma^2}} \d \vec z
 \\
 &\leq
 e^{- \frac{\gamma^2}{2 \sigma^2}} 
 \pr_{\vec z \sim \normal(\pi_C(\x), \sigma^2 I)}[\vec z \in C]
 \leq 
 e^{- \frac{\gamma^2}{2 \sigma^2}} 
 \,,
\end{align*}
where for the penultimate inequality we used the margin assumption which implies that $\x$ must lie $\gamma$-far from the boundary of $C$.
\end{proof}
We are now ready to state and prove our main theorem about agnostic learning with geometric margin $\gamma$.

\begin{theorem}[Learning intersections of Halfspaces with $\gamma$-margin]
    Let $D$ be a distribution on $\R^{d}\times \{\pm 1\}$ where the $\x$-marginal is bounded in the unit ball. Let $\F$ be the class of intersections of $k$ halfspaces. Then, there exists an algorithm that learns $\F$ in the $\gamma$-margin setting that takes $N=k^{\poly(\log k/\epsilon\gamma)}\log(1/\delta)$ samples, runs in time $\poly(d,N)$ and with probability at least $1-\delta$, over the samples it holds that
    \[
    \Pr_{(\x,y)\sim D}\leq \inf_{f\in \F}\E_{(\x,y)\sim D)}[\sup_{\twonorm{\vu}\leq \gamma}\1\{f(\x+\vu)\neq y\}]+\epsilon\,.
    \]
\end{theorem}
\begin{proof}
    Let $f^*$ be the optimal hypothesis that minimizes the $\gamma$-margin error. From \Cref{proposition:margin-invariance},  we have that $\sup_{\x\notin \partial_{\gamma}f^*}\mathbb{S}(\x;\sigma,f)\leq \epsilon/2$ when $\sigma=\gamma/\sqrt{2\log(2k/\epsilon)}$.  The result is then implied by \Cref{thm:random_proj_bounded} and \Cref{lem:margin_smooth}.
\end{proof}
\begin{theorem}[Learning $\F(k,\Gamma)$ with $\gamma$-margin]
    Let $D$ be a distribution on $\R^{d}\times \{\pm 1\}$ where the $\x$-marginal is bounded in the unit ball. Then, there exists an algorithm that learns $\F$ in the $\gamma$-margin setting that takes $N=k^{\poly(k/\epsilon\gamma)}\log(1/\delta)$ samples, runs in time $\poly(d,N)$ and with probability at least $1-\delta$, over the samples it holds that
    \[
    \Pr_{(\x,y)\sim D}\leq \inf_{f\in \F(k,\Gamma)}\E_{(\x,y)\sim D)}[\sup_{\twonorm{\vu}\leq \gamma}\1\{f(\x+\vu)\neq y\}]+\epsilon\,.
    \]
\end{theorem}
\begin{proof}
Let $f^*$ be the function that achieves optimal $\gamma$-margin error. From \Cref{lemma:margin_general}, we have that $\sup_{\x\notin\partial_{\gamma}f^*}\mathbb{S}(\x;\sigma,f)\leq \epsilon/2$ when $\sigma=\gamma/\sqrt{k\log(2/\epsilon}$. Our result is then implied by \Cref{lem:margin_smooth} and \Cref{thm:random_proj_bounded}. 
\end{proof}
\renewcommand{\arraystretch}{2}
\begin{table}[ht]
\centering
\begin{tabular}{|c|c|c|c|}
\hline
Concept Class & Runtime & Model & Source \\
\hline
\hline

Intersections  of $k$ halfspaces  & $d\cdot k^{\tilde{O}(k/\epsilon\gamma^2)}$ & Agnostic & \cite{ArriagaVempala:99}\\
\hline
Intersections of $k$ halfspaces & $\poly(d)\cdot k^{\poly(\log k/(\gamma\epsilon))}$ & Agnostic & This work \\
\hline

Arbitrary functions of $k$ halfspaces & $\poly(d)\cdot k^{\poly(k/(\gamma\epsilon))}$ & Agnostic & This work \\
\hline
$k$ dimensional Convex sets  & $\poly(d)\cdot k^{\poly(k/(\gamma\epsilon))}$ & Agnostic & This work \\
\hline
\end{tabular}
\caption{Our results on distributions with geometric margin $\gamma$}
\label{tab:margin}
\end{table}

\subsection{Distribution Specific Learning} 
In this section, we study the classic setting of agnostic learning with respect to specific distributions. Perhaps surprisingly, our smoothed learning model implies various new results in this standard model. 

First, we prove the following lemma that connects the smoothed error to the true error of a function.
\begin{lemma}[From Distribution Specific Agnostic to Smooth Agnostic]
\label{obs:agnostic_to_smooth}
Fix some boolean function $f:\R^d\mapsto \{\pm 1\}^d$ and some distribution $D$ over labeled
examples on $\R^d \times \{ \pm 1\}$.
It holds
\[
        \err_\sigma(f,D)
        \leq 
         \err(f,D) 
         + \E_{\x\sim D_{\x}}
         [\mathbb S(\x;\sigma, f)]
         \,.
\]
\end{lemma}
\begin{proof}
    A triangle inequality implies that for any vectors $\x,\z$ and label $y$, we have that $\1\{y\neq f(\x+\sigma\z)\}\leq \1\{f(\x)\neq y\}+\1\{f(\x)\neq f(\x+\sigma\z)\} $. Taking expectation of $(\x,y)$ over $D$ and expectation of $\z$ over $\Gauss(0,I_d)$ implies the result.
\end{proof}
Henceforth, we refer to the quantity $\E_{\x\sim D}[\mathbb{S}(\x;\sigma,f)]$ as the \textit{expected sensitivity} of the function $f$ with respect to distribution $D$(or \textit{expected sensitivity} of $f$ for brevity).

\subsubsection{Learning with Anti-Concentration and Concentration}
In this section, we consider the problem of agnostic learning when the marginal has both concentration and anti concentration. In particular, we define the following notion of anti concentration.
\begin{definition}[$M$-anti-concentrated distributions]
    We say that a distribution $D$ on $\R^d$ is $M$-anti-concentrated if for all $\twonorm{\vv}=1$ and continuous intervals $T\subseteq \R$, we have  
        \(
          \pr[\x\cdot \vv\in T ]\leq \tau|T|.
        \)
\end{definition}

We consider $M$-anti-concentrated $(\alpha,\lambda)$-strictly subexponential distributions in this section. Let $\F$ be the class of arbitrary functions of $k$ halfspaces. We prove that we can bound the expected sensitivity of functions in $\F$ with respect to anti concentrated distributions.

\begin{lemma}[Expected Gaussian Sensitivity of functions of halfspaces with anticoncentration]
\label{lem:expected_sens_anticonc_halfspace}
    Let $D$ be an $M$-anti concentrated distribution on $\R^{d}$ and let $f$ be a boolean function on $k$ halfspaces. Then, for any $\epsilon>0$, we have that $\E_{\x\sim D}[\mathbb{S}(\x;\sigma,f)]\leq k(2M\epsilon+ e^{-\epsilon^2/(2\sigma^2)}) $.
\end{lemma}
\begin{proof}
    Let $f$ be the function $f(\x)=g\bigl(\sign(\vw_1\cdot\x-b_1),\dots,\sign(\vw_k\cdot \x-b_k)\bigr)$ with $\twonorm{\vw_i}=1$ for all $i\in [k]$ and $g$ being an arbitrary boolean function on $k$ variables. We can assume the bound on the norms of the weights without loss of generality by renormalizing. Let $S$ be the set of all vectors $\vec y\in \R^d$ such that there exists an $i\in[k]$ such that $\vw_i\cdot \vec y\in[b_i-\epsilon,b_i+\epsilon]$. From $M$-anticoncentration of $D$ and a union bound, we have that $\pr_{\x\sim D}[\x\in S]\leq k(2M\epsilon)$.

    We now bound $\mathbb{S}(\x;\sigma,f)$ for $\x\notin S$. By definition of $S$, we have that $|\vw_i\cdot x-b_i|\geq \epsilon$. Thus, using a union bound and the tail bound for the $1$-dimensional normal density, we have that 
    \[
    \pr_{\z\sim \Gauss}[f(\x+\sigma\z)\neq f(\x)]\leq ke^{-\epsilon^2/(2\sigma^2)}.
    \]

    Thus, we have $\E_{\x\sim D}[\mathbb{S}(\x;\sigma,f)]\leq \pr_{\x\sim D}[\x\in S]+ke^{-\epsilon^2/(2\sigma^2)}$ which completes the proof since the first term is bounded by $2kM\epsilon$.
\end{proof}

We are now ready to give an algorithm to learn arbitrary boolean functions of $k$ halfspaces under strictly subexponential distributions with anti-concentration. We use our previous results on smooth learning with respect to strictly subexponential distributions and the bound we proved above on the expected sensitivity of these functions under anti concentrated distributions.
\begin{theorem}[Agnostic Learning of functions of $k$ halfspaces under anti-concentration]
\label{corollary: agnostic_subexp_functions_halfspaces}
Let $k\in \mathbb{Z}_{+}$, $\epsilon,\delta\in (0,1)$ and $M\in \R$. Let $D$ be a distribution on $\R^{d}\times \{\pm 1\}$ such that the marginal distribution is  $M$-anticoncentrated $(\alpha,\lambda)$-strictly subexponential. Let $\F$ be the class of all functions on $k$ halfspaces. There exists an algorithm that draws $N=d^{\poly((\lambda Mk/\epsilon)^{(1+1/\alpha)^3})}\log(1/\delta)$ samples, runs in time $\poly(d,N)$ and computes a hypothesis $h(\x)$ such that, with probability at least $1-\delta$, it holds
\[
    \pr_{(\x,y)\sim D}[y\neq h(\x)]\leq \min_{f\in \F}\pr_{(\x,y)\sim D}[y\neq f(\x)]+\epsilon\,.
\]
    
\end{theorem}
\begin{proof}
    Let $\sigma=(\epsilon/(8\sqrt{2}Mk))\sqrt{\log(8/\epsilon)}$. The algorithm is simple. Run the algorithm from \Cref{thm:smooth_learning_subgaussian} to get a hypothesis $h$ with error at most $\smoothopt+\epsilon/2$. Since $\Gamma(\F)\leq O(k)$, the algorithm uses $N=d^{\poly((\lambda Mk/\epsilon)^{(1+1/\alpha)^3})}\log(1/\delta)$ samples and runs in time $\poly(d,N)$. 
    Using \Cref{lem:expected_sens_anticonc_halfspace} with parameters $\sigma$ and error $\epsilon/8Mk$ implies that $\E_{\x\sim D_\x}[\mathbb{S}(\x;\sigma,f)]\leq \epsilon/2$. From \Cref{obs:agnostic_to_smooth}, we get that the error of the classifier is at most $\min_{f\in \F}\pr_{(\x,y)\sim D}[y\neq f(\x)]+\epsilon$.
\end{proof}

\begin{table}[ht]
\centering
\begin{tabular}{|c|c|c|c|}
\hline
Concept Class & Runtime & Model & Source \\
\hline
\hline
Arbitrary functions on $k$ halfspaces   & $\poly(d)\cdot d^{1/\epsilon^2},k=O(1)$ & Agnostic & \cite{gollakota_testable} \\
\hline
Arbitrary functions on $k$ halfspaces  & $\poly(d)\cdot d^{\poly(k/\epsilon)}$ & Agnostic & This work \\
\hline
\end{tabular}
\caption{Our results on strictly sub-exponential distributions with anticoncentration}
\label{tab:anticonc}
\end{table}

\subsubsection{Learning under Smoothed Distributions}
Finally, we consider learning distribution that have been convolved with a Gaussian. This model was studied in \cite{Kane2013LearningHalfspacesLogConcave}. This is a natural beyond the worst-case model where the input distribution is smoothed to make learnability easier. The problem of agnostically learning functions of halfspaces under smoothed distributions was studied in \cite{Kane2013LearningHalfspacesLogConcave} where they obtained a runtime that was double-exponential in the number of halfspaces. We significantly improve this by reducing the runtime to only depend exponentially on the number of halfspaces.

For any distribution $D$ on $\R^{d}$, we denote the convolution $D\ast \Gauss(0,\tau^2 I_d)$ as $D_{\tau}$. We call this a $\tau$-smoothed distribution. We argue that the expected sensitivity of functions in $\F(k,\Gamma)$ with respect to $\tau$-smoothed distributions strictly subexponential distributions is small .
\begin{lemma}[Bounding Expected Sensitivity of Smoothed Distributions]
\label{lem:sens_gaussian}
Let $f:\R^{d}\mapsto \{\pm 1\}$ be a function in $\F(k,\Gamma)$ and $D$ be an $(\alpha,\lambda)$-strictly subexponential distribution on $\R^{d}$.  Then, we have that $\E_{\x\sim D_{\tau}}[[\mathbb S(\x;\sigma, f)]\leq O(\sigma \Gamma\sqrt{k}/\tau)$.
\end{lemma}
\begin{proof}
Let $f_{\tau}$ be the function defined as $f_{\tau}(\vu)=f(\tau\vu)$.
The quantity we want to bound is equal to $\E_{\x\sim D}\E_{\z_1,\z_2\sim \Gauss}\left[|f_{\tau}((1/\tau)\x+\z_1+(\sigma/\tau)\z_2)-f_{\tau}((1/\tau)\x+\z_1)|\right]$. We bound the inner expectation pointwise for any $\x$. Since $\F(k,\Gamma)$ is closed under shifts and the following argument only relies on the bound on surface area, we assume that $\x=0$ without loss of generality. 

Since $f_{\tau}$ is k dimensional, we know that there exists an orthonormal matrix $\vec P$ such that $f_{\tau}(\x)=f_{\tau}(\vec P\vec P^T \x)$. Let $g$ be the function on $\R^k$ such that $g(\vu)=f_{\tau}(\vec P \vu)$ for some orthonormal matrix $\vec P$. Clearly, $f_{\tau}(\x)=g(\vec P^T \x)$.  Using \Cref{lem:sa_low_dimension}, we get that $\Gamma(g)\leq \Gamma$.  Thus, the quantity we want to bound is $\E_{\z_1,\z_2\sim \Gauss}[ |g(\vec P^T\z_1+(\sigma/\tau)\vec P^T\z_2)-g(\vec P^T\z_1)|]$. Using \Cref{lem:gsa_shift}, we can bound this term by $O\bigl(\Gamma \E_{\z \sim \Gauss(0,I_{k})}[\twonorm{(\sigma/\tau)\z}]]\bigr)\leq O(\sigma \Gamma\sqrt{k}/\tau)$. This completes the proof.
\end{proof}

We also need the following lemma that proves that a $\tau$-smoothed strictly subexponential distribution is also strictly subexponential.
\begin{lemma}
\label{lem:subexp_convolve}
Let $D$ be an $(\alpha,\lambda)$-strictly subexponential distribution on $\R^{d}$ and $C>0$ be a sufficiently large universal constant. Then, for any $\tau>0$, we have that $D_{\tau}$ is $(\min(\alpha,1),C\cdot \max(\lambda,\tau))$-strictly subexponential.
\end{lemma}
\begin{proof}
    Let $s=\min(\alpha,1)$ and let $r=\max(\lambda,\tau)$. For any unit norm vector $\vv$ and $t>0$, we have that $\pr_{\vec y\sim D_{\tau}}[|\vv\cdot \vec y|\geq t]$ is upper bounded by the maximum of the sum of $\pr_{\z\sim \Gauss}[|\vv\cdot \tau \z|\geq t]$ and $\pr_{\x\sim D}[|\vv\cdot \x|\geq t]$ and $1$. Let $f(t)=2\cdot e^{-t^2/2\tau^2}$ and let $g(t)=2\cdot e^{-(t/\lambda)^{1+\alpha}}$. Let $q=\max(1,4r)$. We have that $f(t),g(t)\leq 1/4$ when $t\geq q$. Thus, we have that $\pr_{\vec y\sim D_{\tau}}[|\vv\cdot \vec y\geq q|]\leq 2(f(t)+g(t))\leq 1$. Let $h(t)=2\cdot e^{-(t\ln 2/q)^{1+s}}\geq 2\cdot e^{-\ln 2 (t/q)^{1+s}}$. Clearly, we have that $h(q)\geq 1/2\geq 2\cdot \max{(f(q),g(q))}$. It is straightforward to see that $h$ decreases slower than $2f$ and $2g$. Thus, we have that $2h(t)\geq \max(2\cdot(f(t)+g(t)),1)$ for all $t\geq 0$. This implies that $\pr_{\vec y\sim D_{\tau}}\leq 2\cdot e^{-(t\ln 2/q)^{1+s}}$ for all $t>\geq 0$. This implies that $D_{\tau}$ is $(\min(\alpha,1),(4/\ln 2)\max(\lambda,\tau))$-strictly subexponential.
    
\end{proof}

Now, we are ready to state and prove the main result of this section regarding agnostic learning of functions with bounded surface area under $\tau$-smoothed strictly subexponential distributions.

\begin{theorem}[Agnostic learning under smoothed distributions]
\label{corollary: agnostic_smoothed_subexp}
Let $k\in \mathbb{Z}_{+}$, $\epsilon,\delta\in (0,1)$ and $\tau,\Gamma\in \R_{+}$. Let $D'$ be an $(\alpha,\lambda)$-strictly subexponential distribution on $\R^d$. Let $D$ be distribution on $R^{d}\times \{\pm 1\}$ with marginal ${D}_{\tau}'$. Then, there exists an algorithm that draws $N=d^{\poly\bigl((\lambda k\Gamma/(\tau\epsilon))^{(1+1/\alpha)^3})\bigr)}\log(1/\delta)$ samples, runs in time $\poly(d,N)$ and computes a hypothesis $h(\x)$ such that, with probability at least $1-\delta$, it holds
\[
    \pr_{(\x,y)\sim D}[y\neq h(\x)]\leq \min_{f\in \F(k,\Gamma)}\pr_{(\x,y)\sim D}[y\neq f(\x)]+\epsilon.
\]
    
\end{theorem}
\begin{proof}
    From \Cref{lem:subexp_convolve}, we know that $D'_{\tau}$ is $(\min(\alpha,1), C\max(\lambda,\tau))$-strictly subexponential for large universal constant $C$. For all $f\in \F(k,\Gamma)$, \Cref{lem:sens_gaussian} implies that $\E_{\x\sim D'_{\tau}}[\mathbb{S}(\x;\sigma,f)]\leq \epsilon/2$ when $\sigma=\tau/2\Gamma\sqrt{k}$. Now, we use \Cref{obs:agnostic_to_smooth} and \Cref{thm:smooth_learning_subgaussian} to obtain that there exists an algorithm that draws $N=d^{\poly\bigl((k\Gamma/\tau\epsilon)^{(1+1/\alpha)^3}\bigr)}$, runs in time $\poly(d,N)$ and outputs a hypothesis $h$ that has error at most $\min_{f\in \F(k,\Gamma)}\pr_{(\x,y)\sim D}[y\neq f(\x)]+\epsilon$.
 \end{proof}

\section{Details of Section 3}
\subsection{Some Useful Facts}
\label{subsec:facts}
In this section, we list some facts that will be useful in proving our theorem on polynomial approximation. The proofs are available in \Cref{subsec:proof_facts}.
\begin{lemma}[Approximation of $e^x$ with respect to $\normal(-a^2/2, a^2)$]
\label{claim:approx_exp_pdf_taylor}
Fix $a>0$ and sufficiently large universal constant $C>0$.
Let $p$ be the polynomial $p(x)=\sum_{i=0}^{m-1}\frac{x^{i}}{i!}$ with $m = Ca^2\log(1/\epsilon)$. 
We have that  $\E_{x\sim \normal(-a^2/2,a^2)}[|e^x-p(x)|]\leq \epsilon$.
\end{lemma}
\begin{lemma}
    \label{taylor_exp_under}
 For any $m\in \N$, let $p_m:\R\to \R$ be defined as the degree $m$ Taylor expansion of $e^x$. That is, $p_m(x)=1+\sum_{i=1}^{m}\frac{x^i}{i!}$. Then we have $|p_m(x)|\leq e^{|x|}$ for all $x\in \R$.
\end{lemma}
 \begin{lemma} There exists a large enough universal constant $C$ such that for every $u\in \R$ it holds that
    \label{lemma:ratio_subgaussian}
    \(
    \E_{x\sim \Gauss_1}\left[\left(\frac{\Gauss(x;u,1)}{\Gauss(x;0,1)}\right)^2\right]\leq   e^{{Cu^2}} \,.
    \)
\end{lemma}
  \begin{lemma}
    \label{lem:coeff_bound_comp}
    Let $p_1(x)$ be a polynomial on $\R$ having degree $\ell_1$ and coefficients bounded by $t_1$. Let $p_2(\x)$ be a polynomial on $\R^{d}$ of degree $\ell_2$ with coefficients bounded by $t_2$. The polynomial $q(\x)=p_1(p_2(\x))$ has coefficients bounded by $t_1t_2^{\ell_1}\cdot (Cd)^{2\ell_1 \ell_2}$ for large constant $C$.
\end{lemma}
\begin{lemma}
    \label{lem:coeff_bound_powering}
    Let $p(\vu)$ be a polynomial on $\R^k$ having degree $\ell$ and coefficients bounded by $t$. Then $q(\vu)=(p(\vu))^{m}$ has coefficients bounded by $t^m\cdot (Ck)^{\ell m}$ for large constant $C$.
\end{lemma}
\subsection{Polynomial approximation}
\begin{lemma}[Approximating the Ornstein-Uhlenbeck Smoothed Concept $T_\rho f_{\vec x}(\cdot)$]
    \label{lem:ou_approx_bounded}
    Let $C> 0$ be some large universal constant.
    Let $D$ be a distribution on $\R^d$ with every point $\x$ in the support of $D$ having $\twonorm{\x}$ at most $R$. Let $f:\R^d\to\{\pm 1\}$ and $f_{\x}:\R^d\to\R$ be defined as $f_{\x}(\z)=f(\x+\z)$. Then, for any $\epsilon>0$, there exists a polynomial $p_{\z}$ parameterized by $\z$ such that:
  \begin{enumerate}
            \item It holds that $\E_{\x \sim D}\E_{\z\sim \Gauss}\bigl[(p_{\z}(\x)-T_{\rho}f_{\x}(\z))^2\bigr]\leq \epsilon$,
            \item The degree of $p_{\z}$ is at most $C(R/\rho)^2\log(1/\epsilon)$,
            and  every coefficient of $p_{\z}$ is bounded in absolute value by $d^{C\bigl((R/\rho)^2\log(1/\epsilon)\bigr)^2}$\,. 
    \end{enumerate}
\end{lemma}
\begin{proof}
   We observe that  $T_{\rho}f_{\x}(\z)=\E_{\vec s\sim \Gauss}[f(\x+\sqrt{1-\rho^2}\z+\rho \vec s)]$ has 
   the variable $\x$ inside $f$. Recall that our goal is to construct a polynomial in $\x$ and, since we have no control over $f$ (which can possibly be very hard to approximate pointwise with a polynomial), we decouple $f$ and $\x$ in the expression of $T_{\rho}f_{\x}$ by writing the function as an expectation over a Gaussian centered at $\x/\rho$. 
    \begin{align*}
        T_{\rho}f_{\x}(\z)&=\E_{\vec s\sim \Gauss}\bigl[f(\x+\sqrt{1-\rho^2}\z+\rho\vec s)\bigr]
        =\E_{\vec s\sim \Gauss(\vec x/\rho,\vec I)}\bigl[f(\sqrt{1-\rho^2}\z+\rho\vec s)\bigr],
        \end{align*}
        Next, we can recenter the expectation around zero and express the Ornstein-Uhlenbeck operator as follows:
    \begin{align*}
        T_{\rho}f_{\x}(\z) &=\E_{\vec s\sim \Gauss}\biggl[f(\sqrt{1-\rho^2}\z+\rho\vec s)\cdot \frac{\Gauss(\vec s;\x/\rho,\vec I)}{N(\vec s;\vec 0,\vec I)}\biggr] =
        \E_{\vec s\sim \Gauss}\biggl[
        f(\sqrt{1-\rho^2}\z+\rho\vec s)\cdot e^{-\frac{\twonorm{\x}^2}{2\rho^2}+(\x/\rho)\cdot \vec s}
        \biggr] \,.
    \end{align*}
    To construct our polynomial, we now approximate ${e^{-\frac{\twonorm{\x}^2}{2\rho^2}+ (\x/\rho)\cdot\vec s}}$ using the 1-dimensional Taylor expansion of the exponential function $q(\x,\vec s)=q_m\bigl({-\frac{\twonorm{\x}^2}{2\rho^2}+(\x/\rho)\cdot\vec s}\bigr)$ where $q_{m}(t)=1+\sum_{i=1}^{m-1}\frac{t^i}{i!}$ is the degree $m-1$ Taylor approximation of $e^x$. Thus, our final polynomial $p_{\z}(\x)$ is 
    \[
        p_{\z}(\x)=\E_{\vec s\sim \Gauss}\bigl[f(\sqrt{1-\rho^2}\z+\rho\vec s)\cdot q(\x,\vec s)\bigr]\,.
    \]
    Let $\Delta(\x)$ be defined as the error term $\E_{\z \sim \normal}[|p_{\z}(\x)-T_{\rho}(f_{\x}(\z))|]$. 
    We have that $\Delta(\x)$ is equal to 
    \begin{align*}
    \Delta(\x) &=
    \E_{\z\sim \Gauss}\Big[\E_{\vec s\sim \Gauss}\bigl[|f(\sqrt{1-\rho^2}\z+\rho\vec s)|\cdot 
    \big|q(\x,\vec s)-e^{-\frac{\twonorm{\x}^2}{2\rho^2}+(\x/\rho)\cdot\vec s}\big| \Big]  \Big]
    \\
    &\leq 
    \E_{\vec s\sim \Gauss}\Big[\big|q(\x,\vec s)-e^{-\frac{\twonorm{\x}^2}{2\rho^2}+(\x/\rho)\cdot\vec s}\big| \Big]
    \,,
    \end{align*}
    where for the inequality we used the fact that $|f(\vec x)|=1$ for all $\x$.
    We now observe that when $\vec s \sim \normal$ the random variable 
    $-\|\x\|_2^2/(2 \rho^2) + (\x/\rho) \cdot \vec s$ is distributed as
    $\normal(-\alpha^2/2, \alpha^2)$, where $\alpha = -\|\x\|_2^2/\rho^2$.  
    Therefore, we have reduced the original polynomial approximation problem to showing that 
    the Taylor expansion of the exponential function converges fast in $L_1$ to $e^x$ with 
    respect to $\normal(-\alpha^2/2, \alpha^2)$. For this, we will use \Cref{claim:approx_exp_pdf_taylor}.
    
We set $a=\twonorm{\x/\rho}$ in \Cref{claim:approx_exp_pdf_taylor}. Thus, the degree $m$ of our Taylor expansion is bounded by $O((R/\rho)^2\log(1/\epsilon))$. Thus, we obtain that the final error of our polynomial is
        \begin{align*}
            \E_{\x\sim D}\E_{\z\sim \Gauss_k}\bigl[|(T_{\rho}f_{\x}(\z)-p_{\z}(\x))|\bigr]\leq \E_{\x\sim D}\bigl[\Delta_1(\x)\bigr]\leq \epsilon.
        \end{align*}
        
        The degree of our polynomial $p_{\z}(\vu)$ is $O((R/\rho)^2\log(1/\epsilon))$.  We now bound the coefficients of $p_{\z}(\x)$. Since $f$ is bounded by $1$, it suffices to bound the coefficients of $\x$ in the polynomial $\E_{\x\sim \Gauss}\bigl[q(\x,\vec s)\bigr]$ Since $q(\x,\vec s)=p_e\bigl({-\frac{\twonorm{\x}^2}{2\rho^2}+\langle (\x/\rho),\vec s\rangle}\bigr)$, it is a composition of a univariate polynomial $p_e$ and a multivariate polynomial $r(\x,\vec s)={-\frac{\twonorm{\x}^2}{2\rho^2}+\langle (\x/\rho),\vec s\rangle}$. We now use \Cref{lem:coeff_bound_comp} which bounds the coefficients of the composition of polynomials. 
    
From \Cref{lem:coeff_bound_comp}, for any fixed $\x$, we have obtained that the coefficients of $q(\vu,\x)$ are bounded by $\norm{\x}_{\infty}^{m}\cdot (C_1k)^{m}$ for large enough constant $C_1$. Thus, the coefficients of $\E_{\x\sim \Gauss}\bigl[q(\vu,\x)\bigr]$ are bounded by $\E_{\x\sim \Gauss}\bigl[\norm{\x}_{\infty}^m\cdot (C_1k)^{m}\bigr]$ which is bounded above by $(C_2d)^{m^2}$ for large enough $C_2$. This follows from the fact that $\norm{\x}_{\infty}^m\leq \sum_{i=1}^{k}|\x_i|^m$ and the fact that $\E_{\x\sim \Gauss}\bigl[|\x_i|^m\bigr]$ is atmost $(C_3 m)^m$ for some constant $C_3$. Thus, we finally conclude that the coefficients of $p_{\z}(\vu)$ are bounded above by $d^{C\bigl((R/\rho)^2\log(1/\epsilon)\bigr)^2}$ for some large constant $C$.
\end{proof}
\subsection{Proofs of Useful Facts}
\label{subsec:proof_facts}
We now prove the facts stated in \Cref{subsec:facts}.

\begin{proof}[Proof of Claim~\ref{claim:approx_exp_pdf_taylor}]
Let $\Delta=\E_{x\sim D}[|e^{x}-p(x)|]=\E_{t\sim \Gauss(0,1)}[|e^{-a^2/2+a t}-p(-a^2/2+a t)|]$.
    We have that ${\Delta}$ can be bounded as the sum of the following two terms:
    \[
        \Delta_{1}=\E_{t\sim \Gauss_1}\Big[2e^{|-\frac{a^2}{2}+at|}\cdot  
    \1\{|t|> T\}\Big] \text{ and } 
            ~\Delta_{2}=\E_{t\sim \Gauss_1}\Big[e^{|-\frac{a^2}{2}+at|}\cdot \frac{\bigl|-\frac{a^2}{2}+at\bigr|^{m}}{(m!)}\cdot 
        \1\{|t|\leq  T\}\Big]\,.\]

        The second term's bound comes from the fact that $|p_e(x)-e^{x}|\leq \frac{e^{|x|}}{m!}\cdot |x|^{m}$.
        The first term's bound follows from the fact that $|p_e(x)-e^{x}|\leq 2e^{|x|}$ which is implied by \Cref{taylor_exp_under}.

    We first bound $\Delta_{1}$. We have that 
    \begin{align*}
        \Delta_{1}&\leq 2\sqrt{\E_{t\sim \Gauss_1}[e^{|-a^2+2at|}]\Pr_{t\sim \Gauss_1}[|t|>T]} 
        \leq 4\sqrt{\E_{t\sim \Gauss_1}[e^{-a^2+2at}]\Pr_{t\sim \Gauss_1}[|t|>T]}\\
    &\leq 4\sqrt{e^{Ca^2-T^2/2}}\leq \epsilon/2\,,
    \end{align*}
    when $T= C'a\log(1/\epsilon)$ for large constant $C'$. The first inequality follows by applying Cauchy-Schwartz. The second follows from the symmetry of the Gaussian distibution. We get the third inequality by using \Cref{lemma:ratio_subgaussian} and bounds on the tail of the Gaussian distribution.

      We now bound $\Delta_{2}$. We have that $\bigl|-\frac{a^2}{2}+at\bigr|$ is atmost $\frac{a^2}{2}+aT$ when $|t|\leq T$. Thus, we obtain that 
      \begin{align*}
        \Delta_{2}&\leq e^{\frac{a^2}{2}+aT} \frac{\bigl(\frac{a^2}{2}+aT\bigr)^{m}}{(m!)}\leq e^{Ca^2}\biggl(\frac{Ca^2\log(1/\epsilon)}{m}\biggr)^{m}\leq \epsilon/2\,, 
        \end{align*} when $m=C''a^2\log(1/\epsilon)$ for large constant $C''$. Thus, the final error of our polynomial is at most $\epsilon$.
\end{proof}
\begin{proof}[Proof of Lemma~\ref{taylor_exp_under}]
\label{proof:taylor_exp_under}
    For any $x\in \R$, $|p_m(x)|\leq p_m(|x|)$.  We prove the claim by induction on $m$. Clearly, $p_0(x)=1\leq e^{|x|}$. Now, by induction, we have $|\frac{\d}{\d x}p_m(x)|=|p_{m-1}(x)|\leq e^{|x|} \leq \frac{\d}{\d x}e^{x}$ where the last inequality holds for all $x>0$. Also, we have $p_m(0)=e^{0}$. Thus, we get that $|p_m(x)|\leq e^{|x|}$.
\end{proof}

\begin{proof}[Proof of Lemma~\ref{lemma:ratio_subgaussian}]
\label{proof:ratio_subgaussian}
The proof is below is a obtained by completing the squares.
    \begin{align*}
    \E_{x\sim \Gauss_1}\left[\left(\frac{\Gauss(x;u,1)}{\Gauss(x;0,1)}\right)^2\right]&=\int_{z\in \R} e^{-\frac{1}{2}\left(2u^2+z^2-4uz\right)}dz\\
    &=\int_{z\in \R} e^{{u^2}}\cdot e^{-\frac{1}{2}({2u}-z)^2}dz
    \leq  e^{{Cu^2}} \,.
      \end{align*}
\end{proof}


\begin{proof}[Proof of Lemma~\ref{lem:coeff_bound_comp}]
\label{proof:coeff_bound_comp}
    Let $p_1(x)=\sum_{i=0}^{\ell_1}c_i x^{i}$. Thus, $q(\x)=\sum_{i=0}^{\ell_1}c_i\left(p_2(\x)\right)^i$. For any monomial of $q$, the contribution to the coefficient from each term in the previous sum is atmost $t_1\cdot t_2^{\ell_1}\cdot (Ck)^{\ell_1 \ell_2}$ for large constant $C$ from \Cref{lem:coeff_bound_powering}. Thus, the coefficients are bounded by $\ell_1 t_1\cdot t_2^{\ell_1}\cdot(Ck)^{\ell_1 \ell_2}\leq t_1t_2^{\ell_1}\cdot (Ck)^{2\ell_1 \ell_2}$.
\end{proof}

\begin{proof}[Proof of Lemma~\ref{lem:coeff_bound_powering}]
    The number of monomials in a polynomial of degree $\ell$ on $\R^{k}$ is at most $\sum_{i=0}^{\ell}\multichoose{k}{i}=\binom{k+\ell}{k}\leq (e\cdot(k+1))^{\ell}$. Expanding $q(\x)=(p(\x))^{m}$ as a sum, we get $(Ck)^{\ell m}$ product terms for large constant $C$. Each of the terms in this sum have coefficients bounded by $t^m$. Since every monomial is formed as a sum of a subset of these terms, we get that the coefficients of each monomial in $q$ is bounded by $t^m\cdot (Ck)^{\ell m}$
\end{proof}
\subsection{Random Projection and Polynomial Regression}

We are now ready to implement the second and third steps in our plan. The algorithmic idea is simple: reduce dimension by applying a random projection and then run polynomial regression in this low dimension space.

\begin{algorithm} 
\caption{Agnostic Learner for Smooth Boolean Concepts with Random Projection}
\label{alg:smooth_boolean_random_projection}
\KwIn{Labeled Dataset $S=\{(\x_i,y_i)\}_{i\in [N]}$, degree $\ell$, subspace dimension $m$}
\KwOut{Hypothesis $h$}

Sample Random Matrix $\vec{R}\in \R^{m\times d}$ with each entry sampled from $\frac{\Gauss(0,1)}{\sqrt{m}}$.\;
Find polynomial $p$ of degree at most $\ell$ such that $p$ minimizes $\frac{1}{N}\sum_{i=1}^{N}|p(\vec{R}\x_i)-y_i|$.\;
Choose $t\in [-1,1]$ such that $\sum_{i=1}^{N}\ind[\sign\bigl(p(\vec{R}\x_i)-t\bigr)\neq y_i]$ is minimized.\;
Output hypothesis $h$, such that $h(\x)=\sign\bigl(p(\vec{R}\x)-t\bigr)$.\;
\end{algorithm}

The above algorithm is run with degree and subspace dimension chosen according to the error we target.  We boost the success probability using a standard technique of repeating the algorithm multiple times and choosing the hypothesis that performs best on an independently chosen validation set. Our choice of the subspace dimension $m$ depends on the intrinsic dimension of our function $f$. The degree $\ell$ we choose depends on the degree bound we obtain from \Cref{lem: boolean_bounded_approx}. For projection matrix $\vec P\in \R^{d\times k}$, let $f(\x)=f(\vec P\vec P^T\x)=g(\vec P^T\x)$ for all $\x$. We argue that with good probability over the random projection matrix $\vec R$, we have that $\twonorm{\vec P^T\x-(\vec P^T\vec R^T\vec R \x)}$ is bounded(\Cref{lem:dim_red_norm}) for all sample points $(\x,y)$ in our input data set. Using this and the bounds on surface area, we argue that $\E_{\z\sim \Gauss}|f(\x+\sigma \z)-f(\vec R^T\vec R\x+\sigma \z)|$ is bounded for all $\x$ in the data set. Thus, we can treat $\vec R\x$ as our inputs and the work in the $m$-dimensional space. From here, the proof is straightforward and uses ideas similar to \cite{KKMS:05}. We now describe in full detail the proof of our main theorem.



\begin{theorem}
\label{thm:random_proj_bounded}
    Let $k \in \mathbb Z_+$ and $\epsilon, \delta,\sigma \in (0,1)$. Let $D$ be a distribution on $\R^d\times \{\pm 1\}$ whose $\x$-marginal is bounded in the unit ball.  There exists an algorithm with sample complexity $N = k^{\tilde{O}\big((\Gamma/\epsilon)^4(1/\sigma^2)\big)}\log (1/\delta)$, runs in time $\poly(d,N)$, and computes a hypothesis $h(\x)$ such that, with probability at least $1-\delta$, it holds
    $ \pr_{(\x, y) \sim \Djoint}[y\neq h(\x)] \le \opt_\sigma + \eps\,$. 
\end{theorem}
\begin{proof}
The algorithm is as follows: run \Cref{alg:smooth_boolean_random_projection} $O\bigl(\log(1/\delta)/\epsilon\bigr)$ times with the same random matrix $\vec R$ and fresh samples each time with parameters $N,\ell,k$ which will be chosen later. Output the hypothesis that has the lowest error on a validation set of size $O\bigl(\log(1/\delta)/\epsilon^2\bigr)$. Clearly the run time is dominated by the time required for polynomial regression which is at most $\poly(d,N)$.

We now argue the correctness of the algorithm. Let $f^*\in \F(k,\Gamma)$ be the optimal function that achieves $\smoothopt$. There exists some orthonormal matrix $\vec P\in\R^{d\times k}$ such that $f^*(\vu)=f^*(\vec P \vec P^T \vu)$ for all $\vu\in \R^{d}$. We say that matrix $\vec R$ is $(\alpha)$-\textit{good} for a set $S$ if $\twonorm{\vec P^T \x-\vec P^T\vec R^T\vec R\x}\leq \alpha$ for all $\x\in S$. We crucially use the following claim which we prove in the end of this section.
\begin{lemma}
\label{lem:dim_red_norm}
     Let $S\subseteq \R^d$ with $\twonorm{\x}\leq B$ for all $\x\in S$ and $\vec W\in \R^{k\times d}$ with $\twonorm{\vec W}\leq \lambda$. Sample a $m\times d$ random matrix $\vec R$ with every entry sampled from $\frac{\Gauss(0,1)}{\sqrt{m}}$  where $m=O\bigl((B\lambda)^2k\log (|S|/\delta)/\epsilon^2\bigr)$. Then, with probability $1-\delta$, we have that $\forall \x\in S$, $\twonorm{\vec W\x - (\vec W\vec R^T\vec R\x)}\leq \epsilon$.
\end{lemma}

We choose the dimension $m$ in such a way that \Cref{lem:dim_red_norm} implies with probability at least  $1-(\delta\epsilon/16)$ that the random gaussian matrix $\vec R$ is $(\epsilon/(32\Gamma))$-\textit{good} for a given set $S$ of size $N$. Having chosen $m$ in this way, it is easy to see that with probability at least $1-(\delta/2)$, the random matrix $\vec R$ is $(\epsilon/(32\Gamma))$-\textit{good} for at least $1-(\epsilon/8)$ mass of the datasets $S$ drawn from $D^{\otimes N}$. Henceforth, we assume that $\vec R$ satisfies the above property. From here, our analysis is similar to the proof of Theorem $5$ of \cite{KKMS:05} accompanied with applications of \Cref{lem:dim_red_norm} and \Cref{lem: boolean_bounded_approx}. 

Consider the sample dataset $S=\{(\x_i,y_i)\}_{i\in [N]}$ of size $N$ in a single run of \Cref{alg:smooth_boolean_random_projection}. Let $p_S$ be the polynomial chosen by the algorithm and let $h_S$ be the corresponding hypothesis that the algorithm outputs. From the proof of Theorem 5 of \cite{KKMS:05}, we have that 
$\frac{1}{N}\sum_{i=1}^{N}\ind[h_S(\x_i)\neq y_i]\leq \min\bigl(\frac{1}{2N}\sum_{i=1}^{N}|p_S(\vec R\x_i)-y_i|,1\bigr)$. We now bound $\E_{S\sim D^{\otimes N}}\left[\ind[h_S(\x_i)\neq y_i]\right]$ by bounding the expectation of the right hand side. We have that  
\begin{align*}
\min&\biggl(\frac{1}{2N}\sum_{i=1}^{N}|p_S(\vec R\x_i)-y_i|,1\biggr)
\leq \min\biggl(\E_{\z\sim \Gauss}\left[\frac{1}{2N}\sum_{i=1}^{N}|p_{\z}(\vec R\x_i)-y_i|\right],1\biggr)
\\
&\leq \underbrace{\min\biggl(\E_{z\sim \Gauss}\left[\frac{1}{2N}\sum_{i=1}^{N}|p_{\z}(\vec R\x_i)-f^*(\x_i+\sigma \z)|\right],1\biggr)}_{\Delta_1(S)}+\underbrace{\E_{\z\sim \Gauss}\left[\frac{1}{2N}\sum_{i=1}^{N}|f^*(\x_i+\sigma \z)-y_i|\right]}_{\Delta_2(S)}.
\end{align*}
The first inequality follows from the fact that $p_S$ is the minimizer of the error and thus beats any polynomial $p_{\z}$ which we choose later. The second is a triangle inequality.
We bound the two terms separately.  For any function $f$, let $f_{\sigma}$ be the function defined as $f_{\sigma}(\x)=f(\sigma\x)$. Let $g^*$ be the function on $\R^{k}$ defined as $g^*(\vec u)=f^*(\vec P \vec u)$. Recall that $f^*(\x)=f^*(\vec P\vec P^T\x)=g^*(\vec P^T\x)$. We use the following claim to bound the surface area of $g^*$ by $\Gamma$. The proof is available in the end of this section. 
\begin{lemma}
    \label{lem:sa_low_dimension}
    Let $f:\R^{d}\to\{\pm 1\}$ be a function such that $f(\x)=f(\vec P\vec P^T \x)$ for some orthonormal matrix $\vec P\in \R^{d\times k}$. Define $g:\R^k\to \{\pm 1\}$ to be the function $g(\vec y)=f(\vec P\x)$. Then, we have that $\Gamma(g)=\Gamma(f)$.
\end{lemma}
First consider $\Delta_1(S)$. We define the terms \[\Delta_{11}(S)=\E_{z\sim \Gauss}\left[\frac{1}{2N}\sum_{i=1}^{N}|p_{\z}(\vec R\x_i)-f^*(\vec R^T \vec R\x_i+\sigma\z)|\right]\] and \[\Delta_{12}(S)=\E_{z\sim \Gauss}\left[\frac{1}{2N}\sum_{i=1}^{N}|f^*( \vec R^T \vec R\x_i+\sigma\z)-f^*(\x_i+\sigma \z)|\right]\] to help us bound $\Delta_1(S)$.

Observe that $\Delta_{11}(S)=\E_{z\sim \Gauss}\left[\frac{1}{2N}\sum_{i=1}^{N}|p_{\z}(\vec R\x_i)-g^*(\vec P^T\vec R^T \vec R\x_i+\sigma\vec P^T\z)|\right]$.
Conditioning on the event that $\vec R$ is $(\epsilon/(32\Gamma))$-\textit{good} for $S$, we have that $\twonorm{\vec P^T \vec R^T\vec R \x_i}\leq \twonorm{\vec P^T \x_i}+\epsilon$ for all $i\in [N]$. Recall that $g^*(\vec y+\sigma \z)=g_{\sigma}^*(\vec y/\sigma+\z)$. Using the fact that $\F(k,\Gamma)$ is closed under scaling, we have that $\Gamma(g^*_{\sigma})\leq \Gamma$. Thus, using \Cref{lem: boolean_bounded_approx}, we obtain a polynomial $q_{\z}$ of degree $\ell=O\bigl((\Gamma/\epsilon)^4(1/\sigma^2)\log(1/\epsilon))$ such that $\E_{\x\sim D_S}\E_{\z\sim \Gauss}\bigl[|g_{\sigma}^*\bigl((\vec P^T\vec R^T \vec R/ \sigma) \x+\vec P^T\z\bigr)-q_{\z}(\vec P^T\vec R^T\vec R \x/\sigma)|\bigr]\leq \epsilon/4$ where $D_S$ is the uniform distribution over the samples $S$. Let $p_{\z}$ be the polynomial $p_{\z}(\vu)=q_{\z}(\vec P^T\vec R^T\vu/\sigma)$. Clearly, $p_{\z}(\vec R\x)=q_{\z}(\vec P^T\vec R^T\vec R\x/\sigma)$. Thus, we obtain that $\Delta_{11}(S)\leq \epsilon/8$.

Observe that $\Delta_{12}(S)$ is equal to $\E_{z\sim \Gauss}\left[\frac{1}{2N}\sum_{i=1}^{N}|g^*( \vec P^T \vec R^T \vec R\x_i+\sigma\z)-g^*(\vec P^T\x_i+\sigma \z)|\right]$. We use the following lemma(proof in the end of this section) along with the fact that $\vec R$ is $(\epsilon/(32\Gamma))$-\textit{good} for $S$ to obtain that $\Delta_{12}(S)\leq \epsilon/4$. 
\begin{lemma}
\label{lem:gsa_shift}
    Let $\F$ be a class of binary functions whose decision boundaries that are $(d-1)$-rectifiable (see \cite{federer1969geometric}), such that $\F$ is closed under arbitrary translations, and its elements have Gaussian Surface area bounded by $\Gamma$. Then, for $\vu,\vv\in \R^d$, we have $\E_{\z\sim \Gauss}\bigl[|f(\vu+\z)-f(\vv+\z)|\bigr]\leq 8\cdot\Gamma\cdot\twonorm{\vu-\vv}$.
\end{lemma}

\begin{remark}
    The notion of rectifiability that we used in the premise of \Cref{lem:gsa_shift} is a notion of regularity which ensures that the decision boundaries of concepts in class $\F$ behave effectively like countable unions of surfaces defined by Lipschitz mappings from $\R^{d-1}$ to $\R^d$. This is a standard notion of non-degeneracy in geometric measure theory (see \cite{federer1969geometric}) and capture most of the classes typically considered in learning theory, like halfspace intersections, polynomial threshold functions and convex sets.
\end{remark}

From a triangle, inequality, it follows that $\Delta_1(S)\leq \min\bigl(\Delta_{11}(S)+\Delta_{12}(S),1\bigr)$. Thus, we get that $\Delta_1(S)\leq (3\epsilon/8)$ when $\vec R$ is $(\epsilon/(32\Gamma))$-\textit{good} for $S$ and at most $1$ otherwise.
Also recall that the second event happens with probability at most $\epsilon/8$ over $S$. Thus, we have that $E_{S\sim D^{\otimes N}}[\Delta_1(S)]\leq 3\epsilon/8+\epsilon/8\leq \epsilon/2$.
Thus, we have that
\begin{align*}
    \E_{S\sim D^{\otimes N}}&\biggl[\min\biggl(\frac{1}{2N}\sum_{i=1}^{N}|p_S(\vec R\x_i)-y_i)|,1\biggr)\biggr]\\&\leq \E_{S\sim D^{\otimes N}}[\Delta_1(S)+\Delta_2(S)]\leq \epsilon/2+\E_{S\sim D^{\otimes N}}[\Delta_2(S)]
\leq \smoothopt+\epsilon/2.
\end{align*}
 The final inequality follows from the definition of $\smoothopt$. Thus, we have that 
\[
\E_{S\sim D^{\otimes N}}\biggl[\frac{1}{N}\sum_{i=1}^{N}\ind[h_S(\x_i)\neq y_i]\biggr]\leq \smoothopt+\epsilon/2.
\] Since our hypothesis $h_S$ is a PTF of degree $\ell$ on $m$ variables, VC theory tells us that for $N=\poly(m^{\ell}/\epsilon)$, we have that 
\[
\E_{S\sim D^{\otimes N}}\biggl[\pr_{(\x, y) \sim \Djoint}[y\neq h_S(\x)]\biggr]\leq \smoothopt+3\epsilon/4.
\] 

By an application of Markov's inequality  we have that with probability at most $\frac{\smoothopt+(3\epsilon/4)}{\smoothopt+(7\epsilon/8)}$, it holds that 
\[
\pr_{(\x, y) \sim \Djoint}[y\neq h_S(\x)]\geq \smoothopt+7\epsilon/8.
\] Thus, with probability $1-\frac{\smoothopt+(3\epsilon/4)}{\smoothopt+(7\epsilon/8)}\geq \epsilon/16$  over samples $S$, we have that $\pr_{(\x, y) \sim \Djoint}[y\neq h_S(\x)]\leq \smoothopt+ 7\epsilon/8$. Let $h_1,h_2,\ldots, h_r$ be the $r$ hypotheses outputted on the $r=O(\log(1/\delta)/\epsilon)$ repetitions of algorithm $2$. With probability at least $1-\delta/4$, there exists $i\in [r]$ such that  $\pr_{(\x, y) \sim \Djoint}[y\neq h_i(\x)]\leq \smoothopt+ 7\epsilon/8$. Thus, using the validation set of size $O(\log(1/\delta)/\epsilon^2)$, with probability at least $1-\delta/4$, we choose a hypothesis $h$ such that $\pr_{(\x, y) \sim \Djoint}[y\neq h(\x)]\leq \smoothopt+ \epsilon$. Thus with total error probability of at most $\delta$, our algorithm outputs a hypothesis $h$ such that
\[
\pr_{(\x,y)\sim D}[y\neq h(\x)]\leq \smoothopt+\epsilon.
\]

We now calculate the required parameters. The degree $\ell$ is $O\bigl((\Gamma/\epsilon)^4(1/\sigma^2)\log(1/\epsilon)\bigr)$. The dimension $m$ is $O\bigl(\frac{k\Gamma^2\log N}{(\epsilon\sigma)^2}\log(k/(\delta\epsilon))\bigr)$ and the number of samples $N=\poly(m^{\ell}/\epsilon)$. We get that these conditions are satisfied when $N=\poly\biggl(\bigl(k\Gamma^2/(\sigma\epsilon)^2\bigr)^{\ell}\biggr)$ and $m\geq \poly\biggl(\frac{k\Gamma^2\log(k/(\delta\epsilon))}{(\epsilon\sigma)^2}\biggr)$.
\end{proof}

\begin{lemma}[\cite{ArriagaVempala:99}]
\label{thm:ArriagaVempala:99}
    Let $S\subseteq \R^d$ with $\twonorm{\x}\leq 1$ for all $\x\in S$ and $\vw\in \R^{d}$ with $\twonorm{\vw}\leq 1$. Sample a $m\times d$ random matrix $\vec R$ with every entry sampled from $\frac{\Gauss(0,1)}{\sqrt{m}}$  where $m=O\bigl(\log (|S|/\delta)/\epsilon^2\bigr)$. Then, with probability $1-\delta$, we have that $\forall \x\in S$, $|\langle \vw,\x\rangle -\langle \vec R\vw,\vec R\x\rangle|\leq \epsilon$.
\end{lemma}
\begin{proof}[Proof of Lemma~\ref{lem:dim_red_norm}]
\label{proof:lem_dim_red_norm}
    Using \Cref{thm:ArriagaVempala:99} on set $S$ with $\epsilon'=\epsilon/(\sqrt{k}\lambda B)$ and $\delta'=\delta/k$ we get that with probability at least $1-\delta/k$, 
    \[
    |\langle \vec W_i,\x\rangle -\langle \vec R\vec W_i, \vec R\x\rangle|\leq \twonorm{\vec W_i}\twonorm{\x}\epsilon'\leq \epsilon/\sqrt{k}
    \]
    for fixed $i\in [k]$. The first inequality follows from the definition of norm and the second inequality follows from the fact that $\twonorm{\vec W_{i}}\leq \twonorm{\vec W}$ for all $i$. Now, applying a union bound gives us that $|\langle \vec W_i,\x\rangle -\langle \vec R\vec W_i, \vec R\x\rangle|\leq \epsilon/\sqrt{k}$ for all $i\in [k]$. Thus, with probability at least $1-\delta$, we have
    \[
    \twonorm{\vec W \x-\vec W\vec R^T\vec R\x}\leq \sqrt{k}\norm{\vec W \x-\vec W\vec R^T\vec R\x}_{\infty}\leq \epsilon. 
    \]
\end{proof}


\begin{proof}[Proof of Lemma~\ref{lem:sa_low_dimension}]
\label{proof:sa_low_dimension}
Since the Gaussian density is spherically symmetric, we have that the Gaussian surface area is spherically symmetric under rotations. Thus, we can assume that $\vec P^T=\begin{bmatrix}\vec I & \vec 0\end{bmatrix}$ where $\vec I$ is the $k\times k$ identity matrix. We have that $f(\x)=f(\x^k)$ where $\x^k_i=\x_i$ for $i\leq k$ and $0$ otherwise. Let $A_f = \{\x\in\R^d: f(\x)=1\}$. Similarly, define $A_g=\{\vec y\in \R^{k}: g(\x)=1\}$. It is easy to see that $A_f=A_g\times \R^{d-k}$. We are now ready to prove the lemma. For any set $S$, let $S^{\delta}$ denote the set of points at distance at most $\delta$ from $S$. Then, we have that \[\Gamma(f) = \liminf_{\delta \to 0} \frac{1}{\delta} \pr_{\z\sim\Gauss(0,I_d)}\bigr[\z \in A_g^\delta\times \R^{d-k} \setminus A_g\times \R^{d-k}\bigr]=\liminf_{\delta \to 0} \frac{1}{\delta} \pr_{\z\sim\Gauss(0,I_k)}\bigr[\z \in A_g^\delta \setminus A_g\bigr]=\Gamma(g).\]
\end{proof}

\begin{proof}[Proof of Lemma~\ref{lem:gsa_shift}]
\label{proof:gsa_shift}
    Let $g(t) = \E_{\z\sim\Gauss}[|f(\vu+\z)-f(\vu+t\cdot\frac{\vv-\vu}{\|\vv-\vu\|_2}+\z)|]$. We first observe that, if $g$ is differentiable, then $\int_{t=0}^{\|\vu-\vv\|_2} g'(t) \, dt = \E_{\z\sim\Gauss}[|f(\vu+\z)-f(\vv+\z)|]$. Therefore, it suffices to show that $g'$ is differentiable and to bound the quantity $g'(t)$ uniformly over $t\in[0,\|\vu-\vv\|_2]$ by $O(\Gamma)$. 

    Let $A_f = \{\x\in\R^d: f(\x)=1\}$ and $A_f(\vu)$ such that $\x\in A_f(\vu)$ iff $\x+\vu\in A_f$. Recall that we have $f:\R^d\to\{\pm 1\}$ and therefore we may express $g(t)$ as follows, where $\vec w = \frac{\vv-\vu}{\|\vv-\vu\|_2}$.
    \[
        g(t) = 2\pr_{\z\sim\Gauss}[\z\in A_f(\vu)\,\triangle\, A_f(\vu+t\cdot \vec w)] 
    \] where $\triangle$ is the symmetric difference.
    Let $B(t) = A_f(\vu)\,\triangle\, A_f(\vu+t\cdot \vec w)$. Then, we have the following
    \begin{align*}
        g'(t) &= \lim_{\delta\to0} \frac{g(t+\delta)-g(t)}{\delta}\\
        &=\lim_{\delta\to0} \frac{2}{\delta}\cdot \Bigr( \pr_{\z\sim\Gauss}[\z\in B(t+\delta)\setminus B(t)] - \pr_{\z\sim\Gauss}[\z\in B(t)\setminus B(t+\delta)] \Bigr)\,.
    \end{align*}
    In order to bound $|g'(t)|$ (which is an upper bound for $g'(t)$), we bound the quantity corresponding to the first term in the limit $\lim_{\delta\to0} \frac{2}{\delta}\cdot \pr_{\z\sim\Gauss}[\z\in B(t+\delta)\setminus B(t)]$ (and similarly the one corresponding to the other term). In particular, we have $B(t+\delta)\setminus B(t)\subseteq A_f(\vu+t\cdot \vec w)\,\triangle\, A_f(\vu+(t+\delta)\cdot \vec w)$, which implies that $|g'(t)| \le 4\lim_{\delta\to0} \frac{1}{\delta}\pr_{\z\sim \Gauss}[\z\in A_f(\vu+t\cdot \vec w)\,\triangle\, A_f(\vu+(t+\delta)\cdot \vec w)]$. Denote with $A_f^\delta$ the set containing all the points with distance at most $\delta$ from the boundary of $A_f$. Denote with $f_{\vu}$ the function with $f_{\vu}(\x) = f(\vu+\x)$. Since $\F$ is closed under translations and only contains functions with $(d-1)$-rectifiable boundaries (see, e.g., \cite{federer1969geometric} or \cite[Proposition A.3]{Kane11}), we have that 
    \[
        \Gamma\ge \Gamma(f_{\vu+t\vw}) = \lim_{\delta\to 0}\frac{\pr_{\z\sim\Gauss}[\z\in A_{f_{\vu+t\vw}}^\delta]}{2\delta}.
    \]
    Observe now that $A_f(\vu+t\cdot \vec w)\,\triangle\, A_f(\vu+(t+\delta)\cdot \vec w) \subseteq A_{f_{\vu+t\vw}}^\delta$, which therefore implies the desired bound.
\end{proof}

\section{Deferred proofs from Section 4}
\begin{lemma}
    \label{lem:poly_approx_boolean_subexp}
    Let $D$ be a distribution on $\R^k$ that is $(\alpha,\lambda)$-strictly subexponential. Let $f:\R^k \mapsto \{\pm 1\}$ be a boolean function such that
    for all $\vec r \in \R^k$ it holds that 
    the GSA of $f_{\vec r}$  is at most $\Gamma$. Then there exists a polynomial $p_{\z}$ parametrized by $\z$ and large universal constant $C$ such that 
    \begin{enumerate}
        \item The (expected) $L_1$ error of $p_{\z}$, $\E_{\z\sim \Gauss_k}\E_{\vu\sim D}\bigl[|p_{\z}(\vu)-f(\z+\vu)|\bigr]$ is at most $\epsilon$,
        \item The degree of $p_{\z}$ is at most $\left(C\lambda k\Gamma^2\log (1/\epsilon)/\epsilon^2\right)^{64(1+1/\alpha)^3} $,
        \item Every coefficients of $p_{\z}$ is bounded(in absolute value) by $e^{\left(C\lambda k\Gamma^2\log (1/\epsilon)/\epsilon^2\right)^{70(1+1/\alpha)^3}}\,.$
    \end{enumerate}
\end{lemma}

\begin{proof}
   From \Cref{prop:ou_boolean}, we have 
$\E_{\z\sim\mathcal{N}(0,I)}\left[\left|T_\rho f_{\vu}(\z)-f(\z+\vu)\right|\right]\leq  2\sqrt{\pi\rho}\cdot \Gamma$. Choosing $\rho=O(\epsilon^2/\Gamma^2)$ makes this error at most $\epsilon/2$. 

We now approximate $T_{\rho}f_{\vu}$ using a polynomial.We know that 
    $\E_{\vu\sim D}\E_{\z\sim\Gauss}\left[\left|T_\rho f_{\vu}(\z)-f_{\vu}(\z)\right|\right]$ is at most $\sqrt{\E_{\vu\sim D}\E_{\z\sim\Gauss}\left[\left(T_\rho f_{\vu}(\z)-f_{\vu}(\z)\right)^2\right]}$. Thus, using  \Cref{lem:ou_approx_subexp}, we get a polynomial $p_{\z}$ of degree $\left(C\lambda k\log (1/\epsilon)\Gamma^2/\epsilon^2\right)^{64(1+1/\alpha)^3}$ such that
$\E_{\vu\sim D}\E_{\z\sim \Gauss}\bigl[|T_{\rho}f_{\vu}(\z)-p_{\z}(\vu)|\bigr]\leq \epsilon/2$ where $C$ is a large universal constant. Also recall that the coefficients of $p_{\z}$ are bounded by $e^{\left(C\lambda k\log (1/\epsilon)/\rho\right)^{70(1+1/\alpha)^3}}$. From a triangle inequality, we get 
$\E_{\vu\sim D}\E_{\z\sim \Gauss}\bigl[|p_{\z}(\vu)-f(\z+\vu)\bigr]\leq \epsilon$.
\end{proof}

We now construct the polynomial approximator for $T_{\rho}f_{\vu}$ and bound its error.

\begin{lemma}[Polynomial approximation of $T_{\rho}f_{\vu}$]
    \label{lem:ou_approx_subexp}
    Let $D$ be a $(\alpha,\lambda)$-strictly subexponential distribution on $\R^k$ and $f:\R^{k}\to\{\pm 1\}$ be any function. Let $f_{\vu}$ be the function defined as $f_{\vu}(\x)=f(\vu+\x)$. Then, there exists a polynomial $p_{\z}(\vu)$ parametrized by $\z$ and large universal constant $C$ such that 
    \begin{enumerate}
        \item It holds that $\E_{\vu\sim D}\E_{\z\sim \Gauss_{k}}\bigl[(p_{\z}(\vu)-T_{\rho}f_{\vu}(\z))^2\bigr]\leq \epsilon$,
        \item The degree of the polynomial $p_{\z}$ is at most $\left(C\lambda k\log (1/\epsilon)/\rho\right)^{64(1+1/\alpha)^3}$,
        \item Every coefficient of $p_{\z}$ is bounded(in absolute value) by $e^{\left(C\lambda k\log B\log (1/\epsilon)/\rho\right)^{70(1+1/\alpha)^3}}$.
    \end{enumerate}
\end{lemma}
\begin{proof}
    Let $Q$ be the distribution on $\R^k$ with probability distribution function $Q(\x)=\frac{1}{2^k}e^{-\sum{|\x_i|}}$. We have that 
    \begin{align*}
        T_{\rho}f_{\vu}(\z)&=\E_{\x\sim \Gauss_k}\bigl[f(\vu+\sqrt{1-\rho^2}\z+\rho\x)\bigr]\\
        &=\E_{\x\sim \Gauss(\vu/\rho,I)}\bigl[f(\sqrt{1-\rho^2}\z+\rho\x)\bigr]\\
        &=\E_{\x\sim Q}\biggl[f(\sqrt{1-\rho^2}\z+\rho\x)\cdot \frac{\Gauss(\x;\vu/\rho,I)}{Q(\x)}\biggr]\\
        &=e^{-\frac{\twonorm{\vu}^2}{2\rho^2}}\E_{\x\sim Q}\biggl[f(\sqrt{1-\rho^2}\z+\rho\x)\cdot e^{-\frac{\twonorm{\x}^2}{2}-\log Q(\x)}e^{ (\vu/\rho)\cdot\x}\biggr]
    \end{align*}
    where the second equality follows by recentering the distribution of $\x$ and the final equality comes from expanding the ratio of the two probability density functions. 
    
    We now define a polynomial $p_{\z}(\vu)$ approximating $T_{\rho}f_{\vu}(\z)$. To do this, we approximate $e^{-\frac{\twonorm{\vu}^2}{2\rho^2}}$ and $e^{\vu\cdot \x}$ using polynomials in $\vu$. 
    First, we use a polynomial $p_1(\vu)$ to approximate $e^{\left(\frac{-\twonorm{\vu}^2}{2\rho^2}\right)}$. This polynomial is given by the following lemma. We choose the parameters later.
\begin{lemma}[Proof in page \pageref{proof:exp_approx_subexp}]
    \label{thm:exp_approx_subexp}
    Let $b\in \Z_{+}$. Let $D$ be a $(\alpha,\lambda)$-strictly subexponential distribution on $\R^k$. Then, there exists a polynomial $q$ of degree $O\left( (b^2\lambda k\log(1/\epsilon))^{2+2/\alpha} \right)$ such that 
    \begin{enumerate}
        \item The approximation error $\E_{\x\sim D}\left[ \left(q(\x)-e^{\left(-\twonorm{\x}^2\right)}\right)^b\right]$ is upper bounded by  $\epsilon$
    \item Every coefficient of $q$ is bounded(in abolute value) by $k^{O\left((b^{2}\lambda k\log(1/\epsilon))^{2+2/\alpha} \right)}\,.$ 
    \end{enumerate}
\end{lemma}
 Second, to approximate $e^{(\vu/\rho)\cdot\x}$, we use the function $p_2(\vu,\x)=p_e((\vu/\rho)\cdot \x)\1\{\twonorm{\x}\leq T\}$ where $p_e(x)=1+\sum_{i=1}^{m-1}\frac{x^i}{i!}$ is the degree $m-1$ Taylor approximation of $e^x$. We choose $m$ and $T$ later. 
    Thus, our final approximation $T_{\rho}f_{\vu}$ is 
    \[
        p_{\z}(\vu)=p_1(\vu)\underbrace{\E_{\x\sim Q}\left[f(\sqrt{1-\rho^2}\z+\rho\x)\cdot e^{-\frac{\twonorm{\x}^2}{2}-\log Q(\x)}p_2(\vu,\x)\right]}_{\text{$q(\vu)$}}\,.
    \]
    $p_{\z}(\vu)$ is a polynomial in $\vu$ as both $p_1$ and $p_2$ are polynomials in $\vu$ when $\x$ is fixed and the dependence on $\x$ gets marginalized out making $q$ a polynomial as well. The indicator variable in the definition of $p_2$ makes our calculations easier and the analysis cleaner. We now want to bound $E_{\vu\sim D}\E_{\z\sim \Gauss}\bigl[(p_{\z}(\vu)-T_{\rho}f_{\vu}(\z))^2\bigr]$. To help us analyse the error, we define the "hybrid" function $\tilde{p}_{\z}(\vu)$ such that 
    \[
        \tilde{p}_{\z}(\vu)=e^{-\frac{\twonorm{\vu}^2}{2\rho^2}}\E_{\x\sim Q}\left[f(\sqrt{1-\rho^2}\z+\rho\x)\cdot e^{-\frac{\twonorm{\x}^2}{2}-\log Q(\x)}p_2(\vu,\x)\right]\,.
    \]

    We have that 
    \[
    \E_{\vu\sim D}\E_{\z\sim \Gauss_k}\bigl[(T_{\rho}f_{\vu}(\z)-p_{\z}(\vu))^2\bigr]\leq 2\cdot \E_{\vu\sim D}\biggl[\underbrace{\E_{\z\sim \Gauss_k}\bigl[(T_{\rho}f_{\vu}(\z)-\tilde{p}_{\z}(\vu))^2\bigr]}_{\Delta_1(\vu)}+\underbrace{\E_{\z\sim \Gauss}\bigl[(\tilde{p}_{\z}(\vu)-{p}_{\z}(\vu))^2\bigr]}_{\Delta_2(\vu)}\biggr]
    \]
    from the fact that $(a+b)^2\leq 2(a^2+b^2)$. We now bound $\Delta_1(\vu)$ and $\Delta_2(\vu)$ separately. We have that 
    \begin{align*}
        \Delta_1(\vu)&=\E_{\z\sim \Gauss_k}\E_{\x\sim Q}\biggl[ f^2(\sqrt{1-\rho^2}\z+\rho\x)e^{2\left(\frac{-\twonorm{\vu}^2}{2\rho^2}-\frac{\twonorm{\x}^2}{2}-\log Q(\x)\right)}\left(p_2(\vu,\x)-e^{ (\vu/\rho)\cdot \x}\right)^2\biggr]\\
         &=\E_{\z\sim \Gauss_k}\underbrace{\E_{\x\sim Q}\biggl[e^{\bigl(-\frac{\twonorm{\vu}^2}{\rho^2}-\twonorm{\x}^2-2\log \Phi(\x)\bigr)}\bigl(p_2(\vu,\x)-e^{(\vu/\rho)\cdot \x}\bigr)^2\biggr]}_{\tilde{\Delta}_{1}(\vu)}
    \end{align*}
    since $|f(\x)|=1$. Observe that $\tilde{\Delta}_1(\vu)$ can be bounded as the sum of the following two terms.
    \[
        \Delta_{11}(\vu)=\E_{\x\sim Q}\left[e^{\left(\frac{-\twonorm{\vu}^2}{\rho^2}-{\twonorm{\x}^2}-2\log Q(\x)\right)}\frac{e^{2|(\vu/\rho)\cdot\x|}}{(m!)^2}|(\vu/\rho)\cdot \x|^{2m}
    \1\{\norm{\x}\leq T\}\right]\] and 
    \[
      \Delta_{12}(\vu)=  \E_{\x\sim Q}\left[e^{\left(\frac{-\twonorm{\vu}^2}{\rho^2}-\twonorm{\x}^2-2\log Q(\x)\right)}e^{2(\vu/\rho)\cdot \x}
      \1\{\norm{\x}> T\}\right]\] where we used the fact that $|p_e(x)-e^{x}|\leq \frac{e^{|x|}}{m!}\cdot |x|^{m}$ and the fact that $p_2(\vu,\x)=0$ when $\twonorm{\x}\geq T$.

      We first bound $\Delta_{11}$. We have that 
      \begin{align*}
        \Delta_{11}(\vu)&\leq 
            2\cdot \E_{\x\sim Q}\left[e^{\left(-\frac{\twonorm{\vu}^2}{\rho^2}-\twonorm{\x}^2-2\log Q{(\x)}\right)+{2(\vu/\rho)\cdot \x}}\right]\frac{\left(T\twonorm{\vu/\rho}\right)^{2m}}{(m!)^2}  \\
            &\leq 2\cdot\frac{\left(T\twonorm{\vu/\rho}\right)^{2m}}{(m!)^2} \sqrt{\E_{\x\sim Q}\left[\left(\frac{\Gauss(\x;\vu/\rho,I)}{Q(\x)}\right)^4\right]} 
            \leq C_1^ke^{C_1\norm{\vu/\rho}_1}\frac{\left(T\twonorm{(\vu/\rho}\right)^{2m}}{(m!)^2} 
      \end{align*}
      where $C_1$ is a large constant.
      The first inequality follows using the fact that $Q$ is symmetric( we replaced $e^{|(\vu/\rho)\cdot \x|}$ by $2e^{(\vu/\rho)\cdot \x}$) and the second inequality follows from an application of Cauchy Schwartz. The last inequality follows from following claim.
      \begin{lemma}[Proof in page \pageref{proof:ratio_subexp}]
    \label{lem:ratio_subexp}
Define the distribution $Q$ on $\R^k$ with density function $Q(\x) = (1/2)^k 
\e^{-\|\x\|_1}$.  Then, there exists a large universal constant $C$ such that for every vector $\vu$, it holds that
\begin{align*}
    \E_{\x\sim Q}\left[\left(\frac{\mathcal{N}(\x; \vu, \vec I)}{
    Q(\x)}\right)^4\right] \leq 
    C^ke^{C\norm{\vu}_1}\,.
    \end{align*} 
\end{lemma}

      We now compute $\E_{\vu\sim D}\left[\Delta_{11}(\vu)\right]$ as we will need it later.
      \begin{align*}
        \E_{\vu\sim D}\left[{\Delta_{11}(\vu)}\right]&\leq C_1^{k}\sqrt{\E_{\vu\sim D}\left[e^{2C_1\norm{\vu/\rho}_1}\right]\E_{\vu\sim D}\left[\frac{\left(T\twonorm{\vu/\rho}\right)^{4m}}{(m!)^4} \right]}\\
        &\leq C_1^{k}\sqrt{\E_{\vu\sim D}\left[e^{2C_1\norm{\vu/\rho}_1}\right]\sum_{i=1}^{k}k^{m}\E_{\vu\sim D}\left[\frac{\left(T|{\vu_i/\rho}|\right)^{4m}}{(m!)^4} \right]}\\
        &\leq C^{\left(C\lambda k/\rho\right)^{3+3/\alpha}}\left(\frac{C{k}eT\lambda}{ m^{\alpha/1+\alpha}\cdot \rho}\right)^{4m}
        \leq \delta
      \end{align*}
      when $m= (C'T\lambda k/\rho)^{3+3/\alpha}\log(1/\delta)$ where $C'$ is a large enough constant. The first inequality follows from an application of the Cauchy-Schwartz inequality. The second inequality uses the fact that $\twonorm{\vu}\leq \sqrt{k}\norm{\vu}_{\infty}\leq \sqrt{k}\sum_{i=1}^k|\vu_i|$. The third inequality is obtained using \Cref{definition:orstein-uhlenbeck} and the following claim whose proof is available at the end of this section.      
\begin{lemma}
    \label{lem:exp_norm}
    If $D$ on $\R^k$ is $(\alpha,\lambda)$-strictly subexponential, then for any constant $b>0$, we have
    \[
    \E_{\vu\sim D}\left[e^{b\norm{\vu}_1}\right]    \leq C^{\left(Cb\lambda k\right)^{3+3/\alpha} }
    \] for large enough constant $C>0$.
\end{lemma}
      We now bound $\Delta_{12}(\vu)$. 
      \begin{align*}
        \Delta_{12}(\vu)&\leq \sqrt{\E_{\x\sim Q}\left[\left(\frac{\Gauss(\x;\vu/\rho,I)}{Q(\x)}\right)^4\right]\cdot \Pr_{\x\sim Q}[\twonorm{\x}>T]}\\
        &\leq \sqrt{\E_{\x\sim Q}\left[\left(\frac{\Gauss(\vw;\vu/\rho,I)}{Q(\x)}\right)^4\right]\cdot k\cdot  e^{-T/k}}\\
        &\leq C_2^ke^{C_2\norm{\vu/\rho}_1}e^{-T/k}
      \end{align*}
      where the first inequality is Cauchy-Schwartz. The last inequality follows from \Cref{lem:ratio_subexp}. The second inequality follows from the following fact about the exponential tail(proof in the end of this section).
\begin{lemma}
\label{lem:exp_tail}
    Let $D$ be the distribution on $\R^k$ with density function $\Phi(\x)=(1/2)^ke^{-\norm{\x}_1}$. We have that \[
        \Pr_{\x\sim D}\left[\twonorm{\x}>T\right]\leq 2k\cdot e^{-T/k}.\]
\end{lemma}

      Thus, we have that 
      \begin{align*}
        \E_{\vu\sim D}\left[{\Delta_{12}(\vu)}\right]&\leq C_2^{k}\cdot \E_{\vu\sim D}\left[e^{C_2\norm{\vu/\rho}_1}\right]e^{(-T/k)}
        \leq C^k\cdot C^{(C\lambda k/\rho)^{3+3/\alpha}}\cdot e^{(-T/k)}
        \leq \delta
      \end{align*}
      when $T= ((C\lambda k/\rho)^{4+4/\alpha}\log (1/\delta) )$.  
      
      Plugging this into the bound for $m$, we get $m\leq (C'\lambda k\log(1/\delta)/\rho)^{15(1+1/\alpha)^2}$ for large constant $C'$. Thus we now get $ \E_{\vu\sim D}\left[\Delta_1(\vu)\right]\leq\E_{\vu\sim D}\bigl[{\Delta_{11}(\vu)+\Delta_{12}(\vu)}\bigr]\leq \epsilon/4$ when $\delta=\epsilon/8$.  
      
      We now bound $\Delta_2(\vu)$. We have that 
   \begin{align*}
    \Delta_2(\vu)&=\E_{\z\sim \Gauss_{k}}\bigl[(\tilde{p}_{\z}(\vu)-{p}_{\z}(\vu))^2\bigr]\\
   &=\E_{\z\sim \Gauss_{k}}\biggl[\bigl(p_1(\vu)-e^{-\frac{\twonorm{\vu}^2}{2\rho^2}}\bigr)^2\cdot \bigl(\E_{\x\sim Q}\bigl[f(\sqrt{1-\rho^2}\z+\rho\x)\cdot e^{-\frac{\twonorm{\x}^2}{2}-\log Q(\x)}\cdot p_2(\vu,\x)\bigr]\bigr)^2\biggr]\\ 
   &\leq  \bigl(p_1(\vu)-e^{-\frac{\twonorm{\vu/\rho}^2}{2}}\bigr)^2\cdot \E_{\z\sim \Gauss_{k}}\biggl[\bigl(\E_{\x\sim Q}\bigl[f(\sqrt{1-\rho^2}\z+\rho\x)\cdot e^{-\frac{\twonorm{\x}^2}{2}-\log Q(\x)}\cdot p_2(\vu,\x)\bigr]\bigr)^2\biggr]
\end{align*} where the last inequality follows since $\vu$ doesn't depend on $\z$. We bound the last term in the product above as
\begin{align*}
    \E_{\z\sim \Gauss_{k}}&\biggl[\bigl(\E_{\x\sim Q}\bigl[ f(\sqrt{1-\rho^2}\z+\rho\x)\cdot e^{-\frac{\twonorm{\x}^2}{2}-\log Q(\x)}\cdot p_2(\vu,\x)\bigr]\bigr)^2\biggr] \\&\leq {\E_{\x\sim Q}\left[e^{-\twonorm{\x}^2-2\log Q(\x)}\right]\cdot \E_{\x\sim Q}\left[\left(1+\sum_{i=1}^{m-1}\frac{\bigl((\vu/\rho)\cdot \x\bigr)^i}{i!}\right)^2\1\{\twonorm{\x}<T\}\right]}\\
    &\leq {2^{2k}\E_{\x\sim Q}\left[e^{-\twonorm{\x}^2+2\norm{\x}_1}\right]\cdot \left(\sum_{i=0}^{m-1}(T\twonorm{\vu/\rho})^i\right)^2}
    \leq C^{k}\cdot \left(\sum_{i=0}^{m-1}(T\twonorm{\vu/\rho})^i\right)^2
\end{align*}
for large enough constant $C$. The first inequality follows from an application of Cauchy Schwartz and the fact that $f$ is boolean. The second inequality comes from expanding $p_2$ and conditioning on the event that $\twonorm{\x}<T$. The last inequality comes from  applying Cauchy Schwartz once more and using \Cref{lem:ratio_subexp} with $\vu$ set to zero.

Thus, we have 
\begin{align*}
    \E_{\vu\sim D}\left[\Delta_2(\vu)\right]&\leq C^{k}\E_{\vu\sim D}\biggl[\bigl(p_1(\vu)-e^{-\frac{\twonorm{\vu/\rho}^2}{2}}\bigr)^2\cdot\biggl(\sum_{i=0}^{m-1}(T\twonorm{\vu/\rho})^i\biggr)^2\biggr]\\
   & \leq C^k\sqrt{\E_{\vu\sim D}\biggl[\biggl(p_1(\vu)-e^{-\frac{\twonorm{\vu/\rho}^2}{2}}\biggr)^4\biggr]\cdot\E_{\vu\sim D}\biggl[\biggl({\sum_{i=0}^{m-1}(T\twonorm{\vu/\rho})^i}\biggr)^{4}\biggr]}\\
   &\leq C^k\cdot \delta \cdot \sqrt{\E_{\vu\sim D}\left[(mT^{m})^{4}\max_{i=0}^{m-1}\twonorm{\vu}^{4i}\right]}\\
    & \leq C^k\cdot \delta\cdot (mT^m)^2(4mk\lambda)^{2m}  \\
    &\leq \epsilon/4
 \end{align*}
when $\delta$ is chosen accordingly. The second inequality is obtained by applying Cauchy-Schwartz.  $p_1(\vu)$ is chosen such that it has approximation error $\delta^2$ when using \Cref{thm:exp_approx} with exponent $4$. The penultimate inequality is obtained by using the fact that $\twonorm{\vu}\leq \sqrt{k}\norm{\vu}_{\infty}$ and then using \Cref{def:strict_subexp}. The degree of $p_1(\vu)$ required to get this error is $O\left((C'm^2k^2\lambda\log(1/\epsilon)/\rho)^{2+2/\alpha}\right)$ where $C'>0$ is a large enough universal constant.

Putting everything together, we get that $E_{\vu\sim D}E_{\z\sim \Gauss}\bigl[(T_{\rho}f_{\vu}(\z)-p_{\z}(\vu))^2\bigr]\leq \epsilon$. The total degree of $p_{\z}(\vu)$ is $\deg(p_1)+\deg(p_2)$ which is at most $\left(C\lambda k\log (1/\epsilon)/\rho\right)^{64(1+1/\alpha)^3} $ for large enough $C$.

We now bound the coefficients of $p_{\z}$. To do this we first recall the definition of $p_{\z}(\vu)$ and observe some properties. Recall that $p_{\z}(\vu)=p_1(\vu)\cdot q(\vu)$ where \[q(\vu)=\E_{\x\sim Q}\left[f(\sqrt{1-\rho^2}\z+\rho\x)\cdot e^{-\frac{\twonorm{\x}^2}{2}-\log Q(\x)}\cdot p_2(\vu/\rho,\x)\right]\,.\] Thus, $p_{\z}(\vu)$ is the product of two polynomials. $p_1(\vu)$ has bounded coefficients as given by $\Cref{thm:exp_approx_subexp}$. We now bound the coefficients of $q(\vu)$. Since $p_2(\vu,\x)=p_e((\vu/\rho)\cdot \x)\1\{\twonorm{\x}\leq T\}$, the $q(\vu)$ term only picks up non zero values when $\twonorm{\x}\leq T$. Note that for each fixed $\vw$, the term inside the expectation is a polynomial in $\vu$. Thus, proving an absolute bound on the coefficients when $\twonorm{\x}\leq T$ will bound the final coefficients of $q(\vu)$. 

We now bound the coefficients of the polynomial $f(\sqrt{1-\rho^2}\z+\rho\x)\cdot e^{-\frac{\twonorm{\x}^2}{2}-\log Q(\x)}\cdot p_2(\vu/\rho,\x)$ where $\x$ is a fixed vector of norm atmost T. Since $\twonorm{\x}\leq T$, we also have that $|\x_i|\leq T$ for all $i\in [k]$. We know that $f$ is bounded by $1$ and $e^{^{\left(-{\twonorm{\x}^2}{2} -\log Q(\x)\right)}}\leq e^{k\norm{\x}_1}\leq e^{k^2T}$. Now, we bound the coefficients of $p_2(\vu,\x)$. This is the composition of two polynomials, $p_e(x)$ and $(\vu/\rho)\cdot \x$. The degree of $p_e$ is $m$ and the coefficients are at most $1$. The degree of $(\vu/\rho)\cdot \x$ is $1$ and coefficients are bounded by $T/\rho$. Applying \Cref{lem:coeff_bound_comp}, we get that the coefficients of $p_2(\vu,\x)$ are bounded by $e^{O\left(mk\log T/\rho\right)}$.  Putting everything together, we obtain that the coefficients of $q(\vu)$ are bounded by $e^{O({k^2m^2/\rho})}$ as $T\leq m$. Putting together the coefficient bounds for $p_1(\vu)$ and $q(\vu)$, we get that the coefficients of $p_{\z}$ are bounded by atmost $e^{\left(C\lambda k\log B\log (1/\epsilon)/\rho\right)^{70(1+1/\alpha)^3}}$ for large constant $C$ and for any $\z$.
\end{proof}

\subsection{\texorpdfstring{Proof of \Cref{thm:smooth_learning_subgaussian}}{Proof of Theorem D.7}}
We give our algorithm here for completeness. We run polynomial regression and the hypothesis we output is a PTF with an appropriately chosen bias term.

\begin{algorithm} 
\caption{Agnostic Learner for Smooth Boolean Concepts}
\label{alg:smooth_boolean}
\KwIn{Labeled Dataset $S=\{(\x_i,y_i)\}_{i\in [N]}$, degree $\ell$}
\KwOut{Hypothesis $h$}

Find polynomial $P$ of degree at most $\ell$ such that $P$ minimizes $\frac{1}{N}\sum_{i=1}^{N}|P(\x_i)-y_i|$.\;
Choose $t\in [-1,1]$ such that $\sum_{i=1}^{N}\ind[\sign\bigl(P(\x_i)-t\bigr)\neq y_i]$ is minimized.\;
Output hypothesis $h$, such that $h(\x)=\sign\bigl(P(\x)-t\bigr)$.\;
\end{algorithm}

We are now ready to state and prove the main theorem of this section.
\begin{theorem}
\label{thm:smooth_learning_subgaussian}
    Let $k \in \mathbb Z_+$ and $\epsilon, \delta,\sigma \in (0,1)$.  Let $D$ be a distribution on $\R^{d}\times \{\pm 1\}$ such that the marginal distribution is $(\alpha,\lambda)$-strictly subexponential. There exists an algorithm that draws $N = d^{\poly\bigl((k\lambda\Gamma/(\sigma\epsilon))^{(1+1/\alpha)^3})\bigr)}$ samples, runs in time $\poly(d,N)$, and computes a hypothesis $h(\x)$ such that, with probability at least $1-\delta$, it holds 
    \begin{equation*}
        \pr_{(\x, y) \sim \Djoint}[y\neq h(\x)] \le     
        \min_{f \in \mathcal F(k, \Gamma)}
         \E_{\z\sim\Gauss}\pr_{(\x, y) \sim \Djoint}[y \neq f(\x + \sigma \z)] + \eps\,.
    \end{equation*} 
\end{theorem}
\begin{proof}

Let $\ell=O\left({\left(\lambda k\Gamma^2\log (1/\epsilon)/(\sigma\epsilon^2)\right)^{64(1+1/\alpha)^3}}\right)$ and $N=\poly(d^{\ell}/\epsilon)$. The marginal distribution of $D$ is denoted by $\D_{\x}$. The algorithm for this task is simple, repeat \Cref{alg:smooth_boolean} for $r=O(\log(1/\delta)/\epsilon)$ times with degree $\ell$ and $N$ fresh samples each time. Output the hypothesis that has the minimum loss on an independent validation set of size $O(\log(1/\delta)/\epsilon^2)$. The time required is the time required for polynomial regression which is $\poly(d,N)$. 

We now analyze the correctness. Let $f^*\in \F(k,\Gamma)$ that be the function that achieves $\smoothopt$. Using the fact that the function is low dimensional, there exists an orthonormal matrix $\vec P$ such that $f^*(\x)=f^*(\vec P \vec P^T \x)$ for all $\x$. This implies that $f^*(\x)=g^*(\vec P \x)$ where $g^*$ is the function on $\R^{k}$ defined as $g^*(\vec u)=f^*(\vec P \vec u)$. Let $f_\sigma$ be defined as the function defined as $f_{\sigma}(\vu)=f(\sigma \vu)$. \Cref{lem:sa_low_dimension} and the fact that $\F(k,\Gamma)$ is closed under scaling imply that $\Gamma(g^*_\sigma)\leq \Gamma$. 
\begin{lemma}[Proof in page  \pageref{proof:subexp_scaling}]
\label{lem:subexp_scaling}
    If $D$ on $\R^d$ is $(\alpha,\lambda)$-strictly subexponential, then for any $\vec A\in \R^{k\times d}$, the distribution of $\vec y=\vec A\x$ when $\x\sim D$ is $\bigl(\alpha,\lambda\twonorm{\vec A}\bigr)$-strictly subexponential.
\end{lemma}
From the above claim, observe that for $\x$ sampled from $\D_{\x}$, the distribution of $\vec P^T \x/\sigma$ is $(\alpha,\lambda/\sigma)$-strictly subexponential on $\R^k$. Thus, \Cref{lem:poly_approx_boolean_subexp} implies that there exists a polynomial $p_{\z}$ with degree at most  $\ell$ and coefficients at most $k^{\poly(\ell)}$ such that 
    \[
        \E_{\x\sim \Dmarginal}\E_{\z\sim \Gauss}\bigl[|p_\z(\vec P^T \x)-g^*_{\sigma}\bigl(\vec P^T\x/\sigma+ \z\bigr)|\bigr]\leq \epsilon/2\,.
        \]
    We first analyze the success probability of one run of the algorithm. Let $S=\{(\x_i,y_i)\}_{i\in [N]}$ be the set of samples and let $p_S, h_S$ be the polynomial and hypothesis output by the algorithm. Repeating the steps from the proof of Theorem~5 in \cite{KKMS:05}, we have that $\frac{1}{N}\sum_{i=1}^{N} \1\{h_S(\x_i)\neq y_i\}\leq \frac{1}{2N}\sum_{i=1}^{N}|p_S(\x_i)-y_i|$. We repeat the argument here for completeness. 
    We want to argue that there exists some threshold $t\in [-1,1]$ such that 
    $\frac{1}{N}\sum_{i=1}^{N} \1\{\sign(p_S(\x_i)+t)\neq y_i\}\leq \frac{1}{2N}\sum_{i=1}^{N}|p_S(\x_i)-y_i|$. In \cite{KKMS:05}, they show that a random threshold works in expectation. To see this, note that for any point $(\x_i,y_i)$, we have that $\sign(p_S(\x_i)-t)\neq y_i$ if and only if $t$ lies between the numbers $p(\x_i)$ and $y_i$. Since $[-1,1]$ has length $2$, the probability that a random $t\sim [-1,1]$ falls between $y_i$ and $p_S(\x_i)$ is $|y_i-p_S(\x_i)|/2$. Thus, in expectation over $t$, it holds that $\frac{1}{N}\sum_{i=1}^{N} \1\{\sign(p_S(\x_i)-t)\neq y_i\}\leq \frac{1}{2N}\sum_{i=1}^{N}|p_S(\x_i)-y_i|$. Since $h_S(\x)=p_S(\x)-t$  for  $t\in[-1,1]$ for which the left hand side is minimized, we have that $h_S$ satisfies the required inequality.
    
      We now bound the right hand side. We have that
    
    \begin{align*}
    \frac{1}{2N}\sum_{i=1}^{N}|p_S(\x_i)-y_i|
    &\leq \frac{1}{2N}\sum_{i=1}^{N}|{p}_{z}(\vec P^T \x_i)-y_i|\\
    &\leq \frac{1}{2N}\sum_{i=1}^{N}|f^*(\x_i+\sigma \z)-y_i|+\frac{1}{2N}\sum_{i=1}^{N}|f^*(\x_i+\sigma\z)-{p}_z(\x_i)|
    \end{align*}
    where $\z$ is an arbitrary vector from $\R^{k}$. The first follows from the fact that $p$ minimizes the empirical error among polynomials of degree less than $\ell$ and the last inequality follows from a triangle inequality. Observe that the expectation(with respect to $\z\sim \Gauss(0,I_d)$ and $(\x,y)\sim D$) of the right hand side of the last equality is $\smoothopt+\epsilon/2$.  
    Thus, we have that 
\[
\E_{S\sim D^{\otimes N}}\biggl[\frac{1}{N}\sum_{i=1}^{N}\1\{h_S(\x_i)\neq y_i\}\biggr]\leq \smoothopt+\epsilon/2.
\] Since our hypothesis $h_S$ is a PTF of degree $\ell$ on $m$ variables, VC theory tells us that for $N=\poly(d^{\ell}/\epsilon)$, we have that 
\[
\E_{S\sim D^{\otimes N}}\biggl[\pr_{(\x, y) \sim \Djoint}[y\neq h_S(\x)]\biggr]\leq \smoothopt+3\epsilon/4.
\] 

By Markov's inequality, we have that with probability at least $\epsilon/16$ over samples $S$, we have that $\pr_{(\x, y) \sim \Djoint}[y\neq h_S(\x)]\leq \smoothopt+ 7\epsilon/8$. Let $h_1,h_2,\ldots, h_r$ be the $r$ hypotheses outputted on the $r=O(\log(1/\delta)/\epsilon)$ repetitions of algorithm $2$. With probability at least $1-\delta/4$, there exists $i\in [r]$ such that  $\pr_{(\x, y) \sim \Djoint}[y\neq h_i(\x)]\leq \smoothopt+ 7\epsilon/8$. Thus, using the validation set of size $O(\log(1/\delta)/\epsilon^2)$, with probability at least $1-\delta/4$, we choose a hypothesis $h$ such that $\pr_{(\x, y) \sim \Djoint}[y\neq h(\x)]\leq \smoothopt+ \epsilon$. Thus with total error probability of at most $\delta$, our algorithm outputs a hypothesis $h$ such that
\[
\pr_{(\x,y)\sim D}[y\neq h(\x)]\leq \smoothopt+\epsilon. \qedhere
\]
\end{proof}

\begin{lemma}
    \label{lem:norm_tail}
    If $D$ on $\R^k$ is $(\alpha,\lambda)$-strictly subexponential, then we have
    \[
    \pr_{\x\sim D}\left[\twonorm{\x}>T\right]\leq 2k\cdot e^{\left(-(T/k\lambda)^{(1+\alpha)}\right) }  \,.
    \]
\end{lemma}
 \begin{proof}
\begin{align*}
    \pr_{\x\sim D}\left[\twonorm{\x}>T\right]&\leq \pr_{\x\sim D}\left[\sum_{i=1}^{k}|\x_i|>T\right]
    \leq \sum_{i=1}^{k}\pr_{\x\sim D}\left[|\x_i|>T/k\right]
    \leq 2k\cdot e^{\left(-(T/k\lambda)^{(1+\alpha)}\right)}\,
\end{align*} where the second inequality follows from a union bound and the last follows from \Cref{def:strict_subexp}.
\end{proof}

\begin{proof}[Proof of Lemma~\ref{lem:ratio_subexp}]
\label{proof:ratio_subexp}
The proof below is a straightforward calculation by completing the squares.
    \begin{align*}
        \E_{\x\sim Q}&\left[\left(\frac{\mathcal{N}(\x; \vu, \vec I)}{
             Q(\x)}\right)^4\right] \\&=\frac{2^{3k}}{{(2\pi)}^{2k}}\int_{\z\in \R^k}e^{-(2\twonorm{\vu}^2+2\twonorm{\z}^2-4\vu\cdot \z-3\norm{\z}_1)} d\z
            =\frac{2^{3k}}{(2\pi)^{2k}}\cdot\prod_{i=1}^{k}\int_{\z_i\in \R}e^{-(2\vu_i^2+2\z_i^2-4\vu_i\z_i-3|\z_i|)} d\z_i\\
            &\leq (C')^{k}\cdot \prod_{i=1}^{k}\int_{\z_i\in \R}e^{-(2\vu_i^2+2\z_i^2-4\vu_i\z_i-3\z_i)} d\z_i
            \leq (C')^ke^{-9k/2}\cdot \prod_{i=1}^{k}\int_{\z_i\in \R} e^{6\vu_i}e^{-(2\sqrt{2}\z_i-(\sqrt{2}{\vu_i}-3/\sqrt{2}))^2}\\
            &\leq C^ke^{C\norm{\vu}_1}
    \end{align*}
    where $C',C$ are appropriately chosen constants. 
\end{proof} 
\begin{lemma}
    \label{lem:exp_ip}
    If $D$ on $\R^k$ is $(\alpha,\lambda)$-strictly subexponential, then for any constant $b>0$ and vector $\vv$ with $\twonorm{\vv}=1$, we have 
    \[
      \E_{\vu\sim D}\left[e^{b|\vv\cdot \vu|  }\right]\leq C^{(Cb\lambda)^{3+3/\alpha}}
    \] for large enough constant $C>0$.
\end{lemma}
\begin{proof}
    \begin{align*}
        \E_{\vu\sim D}\left[e^{b|\vv\cdot \vu|  }\right]&=1+\sum_{i=1}^{\infty}\frac{\E_{\vu\sim D}\left[b^i|\vv\cdot \vu|^i\right]}{i!}
        \leq 1+\sum_{i=1}^{\infty}\frac{b^i\lambda^i (i)^{i/(1+\alpha)}}{(i/e)^i}
        \leq 1+\sum_{i=1}^{\infty}\left(\frac{be\lambda}{i^{\alpha/(1+\alpha)}}\right)^i\\
        &\leq 1+\sum_{i=1}^{(2be\lambda)^{1+1/\alpha}}(be\lambda)^i+\sum_{i=1}^{\infty}\frac{1}{2^i}\leq  1+(be\lambda)^{(2be\lambda)^{1+1/\alpha}+1}+\sum_{i=1}^{\infty}\frac{1}{2^i}
        \leq  C^{(Cb\lambda)^{3+3/\alpha}}.
    \end{align*} for some constant $C>0$. The first equality follows from a Taylor expansion, and the rest are straightforward calculations.
\end{proof}

\begin{proof}[Proof of Lemma~\ref{lem:exp_norm}]
\label{proof:exp_norm}
    \begin{align*}
        \E_{\vu\sim D}\left[e^{b\norm{\vu}_1}\right ] \leq \E_{\vu\sim D}\left[\prod_{i=1}^{k}e^{b|\vu_i|}\right]
        \leq \prod_{i=1}^{k}\left(\E_{\vu\sim D}\left[e^{bk|\vu_i|}\right]\right)^{1/k}
        \leq C^{\left(Cb\lambda k\right)^{3+3/\alpha}}
    \end{align*}
    where the penultimate inequality follows from H{\"o}lder and the last inequality follows from \Cref{lem:exp_ip}.
\end{proof}

\begin{proof}[Proof of Lemma~\ref{lem:exp_tail}]
\label{proof:exp_tail}
    \begin{align*}
        \Pr_{\x\sim D}\left[\twonorm{\x}>T\right]\leq \Pr_{\x\sim D}\left[\sum_{i=1}^{k}|\x_i|>T\right]
        \leq \sum_{i=1}^{k}\Pr_{\x\sim D}\left[|\x_i|>T/k\right]
        \leq 2k\cdot e^{-T/k}
    \end{align*}
    where the last inequality follows from the tail of a univariate exponential random variable.
\end{proof}
\begin{proof}[Proof of Lemma~\ref{lem:subexp_scaling}]
\label{proof:subexp_scaling}
Let $\vv$ be a vector such that $\twonorm{\vv}=1$. Observe that $|\vv\cdot (\vec A\x)|=|(\vec A^T\vv)\cdot \x|\leq \twonorm{\vec A^T\vv}|\vu\cdot \x|\leq \twonorm{\vec A}|\vu\cdot \x|$ where $\vu=\frac{\vec A^T\vv}{\twonorm{\vec A^T\vv}}$. We have that $\pr_{\x\sim D}[|\vv\cdot (\vec A\x)|\geq t]\leq \pr_{\x\sim D}[|\vu\cdot \x|\geq t/\twonorm{\vec A}]\leq 2\cdot e^{-(t/\lambda\twonorm{\vec A})^{1+\alpha)}}$. Considering the second condition from \Cref{def:strict_subexp}, we have that \begin{align*}
    \biggl(\E_{\x\sim D}\bigl[|\vv\cdot \vec A\x|^m\bigr]\biggr)^{1/m}&\leq \twonorm{\vec A}\biggl(\E_{\x\sim D}\bigl[|\vu\cdot\x|^m \bigr]\biggr)^{1/m}
    \leq \twonorm{\vec A}\lambda m^{1/(1+\alpha)}\,.
\end{align*} Finally, we have that
$\E_{\x\sim D}\bigl[e^{(|(\vec A\x)\cdot v|/\lambda\twonorm{\vec A})^{1+\alpha}}\bigr]\leq \E_{\x\sim D}]\bigl[e^{(|\x\cdot v|/\lambda)^{1+\alpha}}\bigr]\leq 2\,.$ 
\end{proof}

We now show that under any $(\alpha,\lambda)$-strictly subexponential distribution, we can approximate the function $e^{-\twonorm{\x}^2}$ by a polynomial of degree $O\bigl(b^{2}\lambda (k\log(1/\epsilon))^{1+1/\alpha} \bigr)$. For this, we use the following theorem from \cite{AA22}.

\begin{lemma}
\label{thm:exp_approx}
    For $T>0$ and error $\epsilon>0$, there exists a polynomial $p$ such that
   \begin{enumerate}
    \item $\sup_{x\in [0,T]}|p(x)-e^{(-x)}|\leq \epsilon $   
    \item $\deg(p)\leq O(\sqrt{T\log(1/\epsilon)})$, if $T=\omega\left(\log1/\epsilon\right)$
    \item $p(x)=\sum_{i=0}^{\deg(p)}c_i x^i$ where $|c_i|\leq e^{\left(C\left(\sqrt{T\log(1/\epsilon)}\right)\right)}$ for all $i\leq \deg(p)$. Here $C$ is a large enough constant.
   \end{enumerate}
\end{lemma}
The bound on the coefficients is not explicitly stated in \cite{AA22} and hence we calculate them.

\begin{lemma}
    Every coefficient of the polynomial $p$ in \Cref{thm:exp_approx} is bounded( in absolute value) by $e^{C\sqrt{T\log(1/\epsilon)}}$
\end{lemma}
\begin{proof}
To bound the coefficients, we first recall their polynomial. Their polynomial $p(t)$ of degree $\ell=\sqrt{T\log(1/\epsilon)}$ approximating $e^{-t}$ is 
\[
p(t)=\sum_{r=0}^{\ell}2^{r-1}A_{r,T/2}Q_r(2t/T-1)    
\] where $Q_r(t)=2^{1-r}\sum_{s=0}^{\lfloor r/2\rfloor}\binom{r}{2s}(t^2-1)^{s}t^{r-2s}$ and $|A_{r,\lamdba}|\leq C^{2r}$ for large constant $C$. The asymptotics of $A_{r,\lambda}$ is explicitly stated in \cite{AA22}. We now bound the coefficients of the polynomial $Q_r(t)$. We bound the coefficient of $t^j$ in each term of the summation for arbitrary $j$. The coefficients of $(t^2-1)^s$ are bounded by $C_1^{s}$ for large constant $C_1$ by using \Cref{lem:coeff_bound_comp}.  From the expression of $Q_r(t)$, we have that each term in the summation contributes $\binom{r}{2s}C_1^{s}\leq C_2^{r}$ for large constant $C_2$. The previous inequality follows from the fact that $s<r$ and $\binom{n}{k}\leq 2^n$. Thus, summing from $0$ to $\lceil r/2\rceil$ bounds the final coefficient of $t^j$ by $C_3^{r}$ for some constant $C_3$. Thus, the coefficients of $Q_r(2t/T-1)$ are bounded by $C_4^r$ for constant $C_4$ by using \Cref{lem:coeff_bound_comp} again. Now since $A_{r,T/2}\leq C^{2r}$, summing from $0$ to $l$ gives us that the coefficients are bounded by $e^{C'\ell}$ for large constant $C'$
\end{proof}

We can now prove our result about approximating $e^{-\twonorm{\x}^2}$ under strictly subexponential distributions.
\begin{proof}[Proof of Lemma~\ref{thm:exp_approx_subexp}]
\label{proof:exp_approx_subexp}
    Let $p=\sum_{i=0}^{\deg(p)}c_ix^i$ be the polynomial obtained from \Cref{thm:exp_approx} with  error $\epsilon/2$ and $T=\omega(\log(1/\epsilon))$ to be chosen later. Our final polynomial is $q(\x)=p\left(\twonorm{\x}^2\right)$. Clearly, $\deg(q)=2\cdot \deg(p)=O(\sqrt{T\log(1/\epsilon)}$. We now bound the error.

    \begin{align*}
        \E_{\x\sim D}\left[ \left(q(\x)-e^{\left(-\twonorm{\x}^2\right)}\right)^b\right]&\leq \epsilon/2+E_{\x\sim D}\left[\left(q(\x)-e^{\left(-\twonorm{\x}^2\right)}\right)^b\1\{\twonorm{\x}^2\geq T\}\right]\\
        &\leq\epsilon/2+ \sqrt{\E_{\x\sim D}\left[\left(q(\x)-e^{(-\twonorm{\x}^2)}\right)^{2b}\right]\cdot  \E_{\x\sim D}\left[\1\{\twonorm{\x}^2\geq T\}\right]}\\
        &\leq \epsilon/2+ \sqrt{2k\cdot \E_{\x\sim D}\left[(|q(\x)|+1)^{2b}\right]\cdot e^{\left(-(\sqrt{T}/k\lambda)^{(1+\alpha)}\right)}} \,.
     \end{align*}
     The first inequality follows from the approximation error of $p$ when $\twonorm{\x}^2\leq T$. We use \Cref{lem:norm_tail} for the last inequality.

     We now bound $\E_{\x\sim D}\left[(|q(\x)|+1)^{2b}\right]$. We have that
     \begin{align*}
        \E_{\x\sim D}[(|q(\x)|+1)^{2b}]&\leq \E_{\x\sim D}\left[\left(1+\sum_{i=0}^{\deg(p)}|c_i|\twonorm{\x}^{2i}\right)^{2b}\right]\\
        &\leq (\deg(p)+1)^{2b}\max_{i=0}^{\deg(p)}|c_i|^{2b}\max_{i=0}^{\deg(p)}\E_{\x\sim D}\left[\twonorm{\x}^{4bi}\right]\\
        &\leq e^{\left(Cb\sqrt{T\log(1/\epsilon)}\right)} \E_{\x\sim D}\left[\left(\sqrt{k}\norm{\x}_\infty\right)^{4b\sqrt{T\log(1/\epsilon)}}\right]\\
        &\leq \sqrt{k}e^{\left(Cb\sqrt{T\log(1/\epsilon)}\right)}  \sum_{i=1}^{k}\E_{\x\sim D}\left[\left(|\x_i|\right)^{4b\sqrt{T\log(1/\epsilon)}}\right]
     \end{align*} for large enough constant $C$.

     The first inequality follows from the definition of $q$ and the second follows from linearity of expectation and a straightforward calculation: $(\deg(p)+1)^{2b}$ is the total number of terms in the summation when expanded, $\max_{i=0}^{\deg(p)}|c_i|^{2b}$ is an upper bound on any coefficient in the expansion and $\max_{i=0}^{\deg(p)}\E_{\x\sim D}[\twonorm{\x}^{4bi}]$ is an upper bound on the expectation of any term in the expansion. The third inequality follows from the fact that $\twonorm{\x}\leq \sqrt{k}\norm{\x}_{\infty}$. The second term in the right hand side of the last inequality above can be bounded by using \Cref{def:strict_subexp}. Putting it all together,we get that  $k\cdot\E_{\x\sim D}\left[(|q(x)|+1)^{2b}\right] e^{\left(-(\sqrt{T}/k\lambda)^{(1+\alpha)}\right)}$ is bounded by 
     $e^{\left(C'\left(b^2 \log \lambda\log k \log\left(T\log(1/\epsilon)\right)\sqrt{T\log(1/\epsilon)}-(T/\sqrt{k\lambda})^{(1+\alpha)/2}\right)\right)}$ where $C'$ is a large enough constant.     

     Choosing $T=O\left(\left(b^2\lambda k\log(1/\epsilon)\right)^{3+3/\alpha}\right)$ makes the total error less than $\epsilon$. Since $T$ is $\omega(\log 1/\epsilon)$, the degree of the final polynomial is $O(\sqrt{T\log(1/\epsilon)})$ which is $O\left( (b^2\lambda k\log(1/\epsilon))^{2+2/\alpha} \right)$.

     We now bound the coefficients of $q$. We have that $q(\x)=p(\twonorm{\x}^2)$ is the composition of two polynomials, $p$ and $\twonorm{\x}^2$. The degree of $p$ is $O(\sqrt{T\log (1/\epsilon)})$ and coefficients bounded by $e^{O\left(\sqrt{T\log(1/\epsilon)}\right)}$. The degree of $\twonorm{\x}^2$ is $2$ and it has coefficients equal to $1$. Thus, using \Cref{lem:coeff_bound_comp} with these polynomials, we get that the coefficients of $q$ are bounded by $k^{O\left( (b^{2}\lambda k\log(1/\epsilon))^{2+2/\alpha} \right)}$.
\end{proof}
    
\section{SQ Lower Bound for Smoothed Agnostic Learning}


\subsection{Background on SQ Lower Bounds}
\label{ssec:SQ-prelims}

Our lower bound applies for the class of Statistical Query (SQ) algorithms.
Statistical Query (SQ) algorithms are a class of algorithms that are allowed
to query expectations of bounded functions of the underlying distribution
rather than directly access samples. Formally, an SQ algorithm has access to the following oracle.

\begin{definition} \label{def:stat-oracle}\label{def:stat}
Let $\cal{D}$ be a distribution on labeled examples supported on $X \times \{-1, 1\}$, for some domain $X$.
A statistical query is a function $q: X \times \{-1, 1\} \to [-1, 1]$.  We define
\textsc{STAT}$(\tau)$ to be the oracle that given any such query $q(\cdot, \cdot)$
outputs a value $v$ such that $|v - \E_{(\vec x, y) \sim \cal{D}}\left[q(\vec x, y)\right]| \leq \tau$,
where $\tau>0$ is the tolerance parameter of the query.
\end{definition}

The SQ model was introduced by Kearns~\cite{Kearns:98} in the context of supervised learning as a
natural restriction of the PAC model~\cite{val84}.  Subsequently, the SQ model has been
extensively studied in many contexts (see, e.g.,~\cite{Feldman16b} and references therein).
The class of SQ algorithms is rather broad and captures a range of known supervised learning
algorithms.  More broadly, several known algorithmic techniques in machine learning are known to be
implementable using SQs. These include spectral techniques, moment and tensor methods, local search
(e.g., Expectation Maximization), and many others (see, e.g.,~\cite{FeldmanGRVX17, DBLP:journals/mor/FeldmanGV21}).
Recent work~\cite{BBHLS21} has shown a near-equivalence between the SQ model and low-degree
polynomial tests.

\paragraph{Statistical Query Dimension}
To bound the complexity of SQ learning a concept class $\cal C$,
we use the SQ framework for problems over distributions~\cite{FeldmanGRVX17}.
\begin{definition}[Decision Problem over Distributions] \label{def:decision}
    Let $\D$ be a fixed distribution and $\mathfrak D$ be a family of distributions.
    We denote by $\mathcal{B}(\mathfrak D, \D)$ the decision (or hypothesis testing) problem
    in which the input distribution $\D'$ is promised to satisfy either
    (a) $\D' = \D$ or (b) $\D' \in \mathfrak D$, and the goal
    is to distinguish between the two cases.
\end{definition}

\begin{definition}[Pairwise Correlation] \label{def:pc}
    The pairwise correlation of two distributions with probability density functions
    $\D_1, \D_2 : \{\pm 1\}^{n} \to \R$ with respect to a distribution with
    density $\D: \{\pm 1\}^n \to \R_+$, where the support of $\D$ contains
    the supports of $\D_1$ and $\D_2$, is defined as
    $\chi_{\D}(\D_1, \D_2) \eqdef |\langle \frac{D_1}{D}-1,\frac{D_2}{D}-1\rangle_D|$ where $\langle f,g\rangle_D=\E_{\x\sim D}[f(\x)g(\x)]$.
\end{definition}
Intuitively, the above quantity is a measure of closeness between two distributions with respect to a base distribution. The smaller the pairwise correlation, the more far apart they are, in some sense.
\begin{definition} \label{def:uncor}
    We say that a set of $s$ distributions $\mathfrak{D} = \{\D_1, \ldots , \D_s \}$
    over $\R^n$ is $(\gamma, \beta)$-correlated relative to a distribution $\D$
    if $|\chi_\D(\D_i, \D_j)| \leq \gamma$ for all $i \neq j$,
    and $|\chi_\D(\D_i, \D_i)| \leq \beta$ for all $i$.
\end{definition}

\begin{definition}[Statistical Query Dimension] \label{def:sq-dim}
    For $\beta, \gamma > 0$ and a decision problem $\mathcal{B}(\mathfrak D, \D)$,
    where $\D$ is a fixed distribution and $\mathfrak D$ is a family of distributions,
    let $s$ be the maximum integer such that there exists a finite set of distributions
    $\mathfrak{D}_\D \subseteq \mathfrak D$ such that
    $\mathfrak{D}_\D$ is $(\gamma, \beta)$-correlated relative to $\D$
    and $|\mathfrak{D}_\D| \geq s.$ The {\em Statistical Query dimension}
    with pairwise correlations $(\gamma, \beta)$ of $\mathcal{B}$ is defined to be $s$,
    and denoted by $\mathrm{SD}(\mathcal{B},\gamma,\beta)$.
\end{definition}
The above quantity measures the maximum size of a maximal set of distributions in the family that have low pairwise correlation. This quantity is used to prove lower bounds for SQ algorithms. At a very high level, it is argued that due to low pairwise correlation, for any statistical query, there exists a response that rules out only a small number of distributions in the class. Thus, since the SQ dimension is large, a large number of queries are required to reduce the search space to one. We refer to \cite{FeldmanGRVX17,reyzin2020statistical} for more details. We use the following lemma from \cite{FeldmanGRVX17}.
\begin{lemma}[Corollary 3.12 of \cite{FeldmanGRVX17}] \label{lem:sq-from-pairwise}
    Let $\mathcal{B}(\mathfrak D, \D)$ be a decision problem, where $\D$ is the reference distribution
    and $\mathfrak{D}$ is a class of distributions. For $\gamma, \beta >0$,
    let $s= \mathrm{SD}(\mathcal{B}, \gamma, \beta)$.
    For any $\gamma' > 0$, any SQ algorithm for $\mathcal{B}$ requires queries of tolerance at most $\sqrt{\gamma + \gamma'}$ or makes at least
    $s  \gamma' /(\beta - \gamma)$ queries.
\end{lemma}

\subsection{SQ Lower Bound on the Dependence on the Dimension and the Smoothing Parameter}

\begin{theorem}[SQ-Lower Bound]\label{thm:sq-noise-variance} 
Fix $d \in \mathbb N$. 
Let $\sigma \in (0, 1)$ and $k \in \mathbb N$ such that  $ \sigma\leq  O(1/\sqrt{\log k})$.
Any SQ algorithm that learns the class of degree $k$ PTFs
in the smoothed agnostic setting (with respect to the uniform distribution on the hypercube) with any accuracy $\epsilon < 1/100$
either requires queries with tolerance at
most $d^{-\Omega(k)}$ or makes at least $d^{\Omega(k)}$ queries.
\end{theorem}

\begin{proof}

Given a subset $S \subseteq \{\pm 1\}^d$,  we denote by $\chi_S(\x)$ the parity function on the subset $S$, i.e.,
$\chi_S(\x) = \prod_{i\in S}\x_{i}$. We can extend the domain of $\chi_S$ to all of $\R^{d}$ as  $\chi_S(\x)=\sign(\prod_{i\in S}\x_{i})$ which is a degree $|S|$ Polynomial Threshold Function (PTF) and hence has surface area $C|S|$ for some universal constant $C>0$. Observe that we have that $\chi_S\in \F(k,Ck)$.

For every subset $S$ of $\{0, 1\}^d$ of size $m$, we define the distribution $D_S$ 
on $\{\pm 1\}^d \times \{\pm 1\}$ to be the distribution
of the pair $(\x, \chi_S(\x))$ where $\x \sim U_d$ is drawn uniformly at random from the Boolean hypercube $U_d = \{\pm\}^d$ (we use $U_d$ to denote both the $d$-dimensional Boolean hypercube and the uniform distribution over it).
Moreover, we define $N$ to be the distribution of $(\x, y)$
where $\x$ is drawn uniformly from the Boolean hypercube
and $y$ is $\pm 1$ with probability $1/2$. 
We let $\mathcal{D}_k$ be the set of all distributions $D_S$
for every subset $S \subseteq \{0, 1\}^d$ of size at most $k$. 
We will show that given a learner for PTFs in the smoothed agnostic
 (under the uniform distribution on the hypercube), we can solve the decision problem $\mathcal{B}(\mathcal D_k, N)$.  
Since all parities are pairwise orthogonal, it is well known that the set of distributions $\mathcal D_k$ is $(0,1)$ correlated.  Therefore, by \Cref{lem:sq-from-pairwise} we obtain that any algorithm that solves 
the decision problem $\mathcal{B}$ either requires 
a query of tolerance $d^{-\Omega(k)}$ or makes at least
$d^{\Omega(k)}$ queries (since the class $\mathcal D_k$
contains $\binom{d}{k}$ distributions).

We now show that using an algorithm $\mathcal{A}$ that 
learns a hypothesis $h(\cdot)$ such that 
\[\pr_{(\x, y) \sim D}[ h( \x) \neq y] \leq 
\E_{\z \sim \normal} \pr_{(\x, y) \sim D}[  \chi_S(\x+\sigma \z) \neq y] + \epsilon\] for some $\epsilon \leq 1/100$ we can solve the $k$-parity decision problem $\mathcal{B}(\mathcal D_k, N)$ defined above.  Let $h(\cdot)$ be the hypothesis returned 
by $\mathcal A$.   We can perform a statistical query
of tolerance $\tau = 1/10$ to obtain an estimate of the
error of $h(\cdot)$, $q = \E_{( \x, y) \sim D}[\1\{h(\x) \neq y\}]$.  If $q \leq 1/2 - 2\tau$ we declare that $D$ corresponds to a $k$-parity, otherwise we declare 
that $D = N$.  We observe that if $D$ actually corresponds to a $k$-parity(with set $S$), then we have that 
\begin{align*}
\E_{\z \sim \normal} \pr_{(\x, y) \sim D}[ \chi_S(\x
+ \sigma  \vec z ) \neq y]
&\leq 
\E_{\z \sim \normal} \pr_{\x \sim U_d}[ \chi_S(\x+\sigma\z)\neq \chi_S(\x) ]
\\
&\leq \max_{\x \in U_d}
\pr_{\z \sim \normal} \Big[ \bigcup_{i=1}^k \Big\{ \sigma |\vec w_i \cdot \vec z| \geq 1/2 \Big\} \Big ] \,,
\end{align*}
where the final inequality follows by the definition of $\chi_S$: for any  $\x \in U_d$ we have that if $\sign(\x_i+\sigma \z_i)\leq 1/2 $ for all $i$ then $\chi_S(\x)=\chi_S(\x+\sigma \z)$.
Using the tail of the Gaussian density, we have that
$
\pr_{\z \sim \normal} \Big[ \bigcup_{i=1}^k \Big\{ \sigma |\vec z| \geq 1/2 \Big\} \Big ]
\leq k \exp(-\Omega(1/(\sigma)^2 ))
\leq k \exp(-\Omega(1/\sigma)^2 )
\leq 1/10$ when $\sigma\leq O(1\sqrt{\log k})$. Therefore, by using a statistical query of tolerance $1/10$ and the learning algorithm $\mathcal{A}$ we can solve the $k$-parity decision problem.  
\end{proof}

\end{document}